\documentclass[10pt,twoside]{article}

\usepackage{fullpage}
\usepackage[utf8]{inputenc} 
\usepackage[T1]{fontenc}    
\usepackage{hyperref}       
\usepackage{url}            
\usepackage{booktabs}       
\usepackage{amsfonts,amsmath,amssymb,amsthm}       
\usepackage{nicefrac}       
\usepackage{microtype}      

\usepackage{epsf}
\usepackage{fancyheadings}
\usepackage{graphics}
\usepackage{graphicx,subfigure}
\usepackage{psfrag}

\usepackage{color}

\usepackage{amsthm}
\usepackage{amsfonts}
\usepackage{amsmath}
\usepackage{amssymb}
\usepackage{mathrsfs}

\usepackage{accents}
\usepackage{listings}

\usepackage{caption}
\usepackage[linesnumbered, ruled, vlined]{algorithm2e}

\usepackage[titletoc,toc,title]{appendix}

\theoremstyle{plain}

\newtheorem{theorem}{Theorem}

\newtheorem{proposition}{Proposition}
\newtheorem{lemma}{Lemma}

\newtheorem{definition}{Definition}

\newtheorem{fact}{Fact}

\newlength{\widebarargwidth}
\newlength{\widebarargheight}
\newlength{\widebarargdepth}

\makeatletter
\long\def\@makecaption#1#2{
        \vskip 0.8ex
        \setbox\@tempboxa\hbox{\small {\bf #1:} #2}
        \parindent 1.5em  
        \dimen0=\hsize
        \advance\dimen0 by -3em
        \ifdim \wd\@tempboxa >\dimen0
                \hbox to \hsize{
                        \parindent 0em
                        \hfil 
                        \parbox{\dimen0}{\def\baselinestretch{0.96}\small
                                {\bf #1.} #2
                                } 
                        \hfil}
        \else \hbox to \hsize{\hfil \box\@tempboxa \hfil}
        \fi
        }
\makeatother

\long\def\comment#1{}

\newcommand{\inprod}[2]{\ensuremath{\langle #1 , \, #2 \rangle}}

\DeclareMathOperator{\var}{var}

\newcommand{\BER}{\ensuremath{\mbox{Ber}}}

\newcommand{\Xspace}{\ensuremath{\mathcal{X}}}

\newcommand{\boldf}{\ensuremath{\mathbf{f}}}
\newcommand{\lvec}{\ensuremath{\mathbf{l}}}
\newcommand{\Lvec}{\ensuremath{\mathbf{L}}}

\newcommand{\xvec}{\ensuremath{\mathbf{x}}}

\newcommand{\wvec}{\ensuremath{\mathbf{w}}}

\newcommand{\EE}{\ensuremath{\mathbb{E}}}

\newcommand{\alphastar}{\ensuremath{\alpha^*}}
\newcommand{\betastar}{\ensuremath{\beta^*}}

\newcommand{\pihat}{\ensuremath{\widehat{\pi}}}
\newcommand{\Phat}{\ensuremath{\widehat{P}}}
\newcommand{\pistar}{\ensuremath{\pi^*}}

\newcommand{\alg}[1]{\textsc{#1}}
\newcommand{\Oh}{\ensuremath{\mathcal{O}}}

\newcommand{\Ind}{\ensuremath{\mathbb{I}}}
\newcommand{\reals}{\ensuremath{\mathbb{R}}}

\newcommand{\gprop}{\ensuremath{g_{\mathsf{prop}}}}
\newcommand{\gunif}{\ensuremath{g_{\mathsf{unif}}}}
\newcommand{\thigh}{\ensuremath{t_{\mathsf{high}}}}
\newcommand{\tlow}{\ensuremath{t_{\mathsf{low}}}}
\newcommand{\wtree}{\ensuremath{w^{(\mathsf{tree})}}}
\newcommand{\ltree}{\ensuremath{l^{(\mathsf{tree})}}}
\newcommand{\Ltree}{\ensuremath{L^{(\mathsf{tree})}}}
\newcommand{\lvectree}{\ensuremath{\lvec^{(\mathsf{tree})}}}
\newcommand{\Lvectree}{\ensuremath{\Lvec^{(\mathsf{tree})}}}
\newcommand{\wvectree}{\ensuremath{\wvec^{(\mathsf{tree})}}}

\newcommand{\quickfigure}[4]{
  \begin{figure}[htbp]
  \centering
  \includegraphics[width=#4]{#3}
  \caption{#2}
  \label{#1}
  \end{figure}
}

\newcommand{\quadfigureexterior}[3]{
\begin{figure}[htbp]
 #3
 \caption{#2}
 \label{#1}
\end{figure}
}

\newcommand{\multifigureexterior}[3]{\quadfigureexterior{#1}{#2}{#3}}

\newcommand{\subfigl}[4]{\subfigure[#3]{\includegraphics[width=#1]{#4} \label{#2}}}
\newcommand{\subfig}[3]{\subfigl{#1}{}{#2}{#3}}

\begin{document}

\begin{center}

{\bf{\LARGE{Best of many worlds: Robust model selection for online supervised learning}}}

\vspace*{.2in}

{\large{
\begin{tabular}{cc}
Vidya Muthukumar$^{\dagger}$ & Mitas Ray$^\dagger$ \\
Anant Sahai$^\dagger$ & Peter L. Bartlett$^{\dagger,\ddagger}$\\
\end{tabular}}}
\vspace*{.2in}

\begin{tabular}{c}
Department of Electrical Engineering and Computer Sciences, UC Berkeley$^\dagger$ \\
Department of Statistics, UC Berkeley$^\ddagger$ \\
\end{tabular}

\vspace*{.2in}

\today

\end{center}
\vspace*{.2in}

\begin{abstract}
  We introduce algorithms for online, full-information prediction that
  are competitive with contextual tree experts of unknown complexity,
  in both probabilistic and adversarial settings.
  We show that by incorporating 
  a probabilistic framework of structural risk minimization into existing
  adaptive algorithms, we can robustly learn not only the presence of stochastic structure 
  when it exists (leading to \textit{constant} as opposed to $\Oh(\sqrt{T})$ regret), but also 
  the correct model order. We thus obtain regret bounds that are competitive 
  with the regret of an optimal algorithm that possesses
  strong side information about both
  the complexity of the optimal contextual tree expert and
  whether the process generating the data is stochastic or
  adversarial.
  These are the first constructive guarantees on simultaneous adaptivity to the model and the presence of stochasticity.
\end{abstract}

\section{Introduction}\label{sec:introduction}

In full-information online learning, there are no generative assumptions on the data. 
We consider \textit{online supervised learning} where we observe pairs of covariates and responses, and need to minimize regret with respect to the best function in hindsight from a \textit{fixed model class}.
In the case where covariates and responses are discrete, we can consider the $0-1$ loss function, and characterize the performance of \textit{tree experts} (also called \textit{contextual experts}) that map a covariate to an appropriate response.
A natural goal is to minimize \textit{minimax cumulative regret} as a function of the number of rounds $T$.
This is well known to scale~\cite{cesa1997use} as $\Oh(\sqrt{T} \cdot \text{(max. model complexity)})$.
Once this is guaranteed, we are especially interested in adaptive algorithms that preserve this guarantee and also adapt to ``easier" stochastic structure.
Again, it is well known that we can get much faster $\Oh(\text{(max. model complexity)})$ rates in this case; essentially, constant regret.
Recent work~\cite{cesa2007improved,erven2011adaptive,de2014follow,luo2015achieving,koolen2015second,koolen2016combining} constructs algorithms that adapt to these faster rates while preserving the minimax rate; thus obtaining the \textit{best of both worlds}.

A more classical goal of adaptivity is \textit{adapting to the complexity of the true model class}.
The traditional \textit{offline} model selection framework~\cite{massart2007concentration} studies a hierarchy of models, and shows that the right model for the problem can be chosen in a \textit{data-adaptive fashion} when the data is generated according to a stochastic iid process.
It is clear that model adaptivity is a natural goal in online learning -- after all, while low regret is important, so is the right choice of benchmark with respect to which to minimize regret.
And the importance of model selection is reflected very naturally in
regret: either our data is not well-expressed by the used model class,
leading us to question what a good regret rate really means,
or our data is well-approximated by simple models and we spend more time than needed looking for the
right predictor, building up unnecessary regret.

In this context, we have a natural goal. Starting with absolutely no
assumptions, we still wish to protect ourselves from adversaries with
the minimax regret rates (up to constants).
However, we also want to adapt simultaneously to the existence \textit{and} statistical complexity of stochastic structure, and perform almost as well as an algorithm with oracle knowledge of that structure would.

Typically, we use adaptive entropy regularization with a changing learning rate to interpolate between the stochastic and adversarial regimes. 
Structural risk minimization has been considered in purely stochastic, or purely adversarial environments, and uses a very different kind of model complexity regularization.
Even in the simplest discrete problems, it was never clear whether these objectives could be achieved simultaneously.
In this paper, we answer the question in the affirmative.
We adaptively recover the stochastic model selection framework in the discrete ``contextual experts" setting and obtain near-optimal, theoretical guarantees on regret in expectation and with high probability.
We also provide simulations to illustrate the value of achieving this kind of two-fold adaptivity.

\paragraph{Our contributions}

We show that an adaptive variant of the \textit{tree expert forecaster} adapts not only to stochastic structure but also the \textit{order} of that stochastic structure that best describes the mapping between covariates and responses.
Our main result is stated informally below.
(For a formal statement of the theorem, see Theorem~\ref{thm:contexttreeadahedge}.)

\textbf{Main theorem (informal):}

\textit{Let $D$ be the maximum model order of tree experts. The regret of our algorithm with respect to the best $d^{th}$-order tree expert is $O(\sqrt T 2^d)$ in an adversarial setting and $O(d \cdot D \ln D \cdot 2^{2d})$ \textbf{with high probability} when the data is actually generated by a $d^{th}$-order tree expert, for any $d\in\{0,\ldots,D\}$.}

Thus, we can recover stochastic online model selection in an adversarial framework -- our regret rate for $d^{th}$-order processes is achieved without knowing the value of $d$ in advance, or even that the process is stochastic.
This rate is competitive with the optimal regret rate that would be achieved by a greedy algorithm possessing side information about both the existence of stochastic structure and the true model order.
We will see the empirical benefit of this two-fold adaptivity in the simulations in Section~\ref{sec:simulations}, where we compare directly to existing algorithms that only achieve one kind of adaptivity.

Interestingly, we are able to obtain these guarantees for an algorithm that is a natural adaptation of the standard exponential weights framework, and our results have an intuitive interpretation. 
We combine the adaptivity to stochasticity of an existing
``best-of-both-worlds" algorithm (called
\alg{AdaHedge}~\cite{erven2011adaptive,de2014follow}) with the prior
weighting on tree experts that is used in tree
forecasters~\cite{helmbold1997predicting}\footnote{Most interestingly,
this prior distribution was designed for the original \textit{tree expert
forecaster}~\cite{helmbold1997predicting}, but this algorithm could not effectively utilize the prior because of the fixed learning rate.}.
As is intuitive, the prior is inversely proportional to the complexity of the tree expert.

Our analysis recovers the stochastic structural risk minimization framework in a probabilistic sense.
There are two penalties involved: the complexity of the model selected (to achieve model selection) as well as determinism (to ensure protection against adversaries).
Remarkably, our algorithm uses a common time-varying, data-dependent learning rate, defined in the elegant \alg{AdaHedge} style, to \textit{learn} the correct proportion with which to apply both regularizers.

\paragraph{Related work}

The framework for \textit{offline} structural risk minimization in purely stochastic environments was laid out in seminal work (for a review, see~\cite{massart2007concentration}). 
Generalization bounds are used to characterize model order complexity, and empirical process theory is used to show that data-adaptive model selection can be performed with high probability.
Online bandit approaches for stochastic model selection have also been considered more recently~\cite{agarwal2011oracle}.

On the other side, the paradigm for \textit{adversarial} regret
minimization was laid out in the discrete ``experts" setting in
seminal work (for a review, see~\cite{cesa1997use}), and subsequently lifted up to the more general \textit{online convex optimization} framework (for a review, see~\cite{shalev2012online}).
The next natural goal was adaptivity to several types of ``easier" instances while preserving the worst-case guarantees.
Most pertinent to our work are the easier \textit{stochastic losses}~\cite{de2014follow}, under which the greedy Follow-the-Leader algorithm achieves regret $\Oh(1)$.
In the experts setting, multiple algorithms have been proposed~\cite{cesa2007improved,erven2011adaptive,de2014follow,luo2015achieving,koolen2015second,koolen2016combining} that adaptively achieve $\Oh(1)$ regret.
Some of these guarantees have been extended to online optimization~\cite{van2016metagrad}.
As we will see, naively extending these analyses to the tree expert forecaster problem gives a pessimistic $\Oh(2^{D})$ regret bound.
In our work, we show that we can get the best of \textit{many worlds} and greatly improve the exponent to $\widetilde{\Oh}(2^{d})$, reducing the dependence on the maximum model complexity $D$ from exponential to linear.

Recent guarantees on adapting to a simpler model class, \textit{but not to stochasticity}, have also been developed~\cite{rakhlin2013online,orabona2014simultaneous,luo2015achieving,koolen2015second,orabona2016coin,foster2017parameter}.
Many of these approaches~\cite{rakhlin2013online,orabona2014simultaneous,orabona2016coin,foster2017parameter} do not improve the $\Oh(\sqrt{T})$ rate for stoachastic data.
Others~\cite{luo2015achieving,koolen2015second} obtain second-order quantile regret bounds in terms of a data-dependent term and the correct model complexity \textit{in the worst case} -- but the subsequent analysis in the stochastic regime~\cite{koolen2016combining} avoids the model selection issue, and again yields a pessimistic $\Oh(2^D)$ regret bound\footnote{We do not believe this to be a shortcoming of the algorithms, to be clear: sharper analysis of their updates would likely yield similar probabilistic model selection guarantees.}.
In our work, we adaptively recover the \textit{stochastic model selection framework} from the \textit{adversarial setting} and obtain sharp, closed-form regret bounds for data generated from a hierarchy of stochastic models.


And so, while the notions of adapting to stochasticity and simpler models have been considered separately in online learning, no previous analysis shows that we can provably simultaneously adapt to both.
This has been proposed as an important objective in recent work~\cite{van2016metagrad,foster2017parameter}.


\section{Problem statement}\label{sec:problem}

We consider an online supervised learning setting over $T > 0$ rounds, in which we receive context-output pairs $(X_t,Y_t)_{t=1}^T$.
We consider $X_t \in \Xspace^D, Y_t \in \Xspace$, where $\Xspace = \{0,1\}$ is the binary alphabet\footnote{As a general note, all our analysis can easily be extended to the $m$-ary case. We present the binary case for simplicity.}.
It will also be natural to consider the truncated version of $X_t$ that only represents the last $d$ coordinates -- we denote this by $X_t(d)$, with the convention that $X_t := X_t(D)$.

We follow the online supervised learning paradigm: before round $t$, we are given access to $X_t$, but not $Y_t$.
Let $\mathcal{F}_D$ denote the set of all \text{tree experts}, expressed as Boolean functions from $\Xspace^D$ to $\Xspace$.
We will also be considering tree experts that map from the subcontexts $\{X_t(h)\}$ to outputs $Y_t$, denoted by $\boldf_h \in \mathcal{F}_h$ for all values of $h$ in $\{0,1,\ldots,D\}$.
We use the shorthand notation $\boldf := \boldf_D \in \mathcal{F}_D$.
We define the \textit{order} of a tree expert, denoted by $\text{order}(\boldf_h)$, as the minimum value of $d \leq h$ for which its functionality can be expressed equivalently in terms of a function from $\Xspace^d$ to $\Xspace$.
That is,
\begin{align}
\text{order}(\boldf_h) := \min\{d \leq h: \text{ there exists } \boldf'_d \in \mathcal{F}_d \text{ s.t. } \boldf_h(x(h)) = \boldf'_d(x(d)) \text{ for all } x(h) \in \Xspace^h\} .
\end{align}

We define our randomized online algorithm for \textit{prediction using tree experts} in terms of a sequence of probability distributions $\{\wvectree_t\}_{t =1}^T$ over the set $\mathcal{F}_D$ of all tree experts.
Note that $\wvectree_t$ cannot depend on $\{(X_s,Y_s)\}_{s \geq t + 1}$ or $Y_t$.
We denote the realization of the prediction at time $t$ by $\widehat{Y}_t \in \Xspace$, and the distribution on $\widehat{Y}_t$ by $\wvec_t$ (clearly induced by $\wvectree_t$).
After prediction, the actual value $Y_t$ is revealed, and the expected loss is modeled as $0-1$ loss depending on whether we get the prediction right.
Formally, we have $\lvec_t = \begin{bmatrix} \Ind[Y_t \neq 0] & \Ind[Y_t \neq 1] \end{bmatrix}$, and the expected loss of the algorithm in round $t$ is given by $\inprod{\wvec_t}{\lvec_t} = w_{t,1-Y_t}$.
We denote as shorthand
\begin{align*}
L_{t,\boldf} &:= \sum_{s=1}^t \Ind[Y_s \neq \boldf(X_t(h))] \text{ for all } \boldf \in \mathcal{F}_h, h \leq D \\
L_{X,t,y} &:= \sum_{s=1}^t \Ind[X_s = X;Y_s \neq y] \text{ for all } X \in \Xspace^h, h \leq D, y \in \Xspace \\
\Lvec_{X,t} &:= \begin{bmatrix} L_{X,t,0} & L_{X,t,1} \end{bmatrix} \text{ for all } X \in \Xspace^h, h \leq D .
\end{align*}

\subsection{Adaptive regret minimization and \alg{ContextTreeAdaHedge$(D)$}}

The traditional quantity of \textit{regret} measures the loss of an algorithm with respect to the loss of the algorithm that possessed oracle knowledge of the best single ``action" to take in hindsight, after seeing the entire sequence offline.
In the context of online supervised learning, this ``action" represents the best $d^{th}$-order Boolean function $\widehat{F}_d(T) \in \mathcal{F}_d$.
The expected regret with respect to the best $d^{th}$-order tree expert is defined as $R_{T,d} := \sum_{t=1}^T \inprod{\wvec_t}{\lvec_t} - L_{T,\widehat{F}_d(T)}$.

Our algorithm is effectively an exponential-weights update on tree experts equipped with a \textit{time-varying, data-dependent learning rate} and a suitable prior distribution on tree experts.
We start by describing the structure of the prior distribution.
\begin{definition}\label{def:preweights}
For any non-negative-valued function $g: \{0,1,\ldots,D\} \to \reals_{+} \cup \{0\}$, we define the prior distribution on all tree experts in $\mathcal{F}_D$, $\wvectree_{1,\boldf}(g) = \frac{\sum_{h = \text{order}(\boldf)}^D g(h)}{Z(g)}$, where $Z(g)$ is the normalizing factor.
\end{definition}

We select a function $g(\cdot)$ and use the prior defined above to effectively downweight more complex experts.
We will see that the choice of prior is crucial to recovering stochastic model selection.

A good \textit{data-adaptive} choice of $\{\eta_t\}_{t \geq 1}$ has been an intriguing question of significant recent interest.
The idea is that we want to learn the correct learning rate for the problem.
We consider a particularly elegant choice based on the algorithm \alg{AdaHedge}, that was defined for the simpler experts setting.
We denote $\eta_{s_1}^{s_2} = \{\eta_s\}_{s=s_1}^{s=s_2}$ for shorthand.
\begin{definition}[\cite{de2014follow}]
The \alg{AdaHedge} learning rate process $\{\eta_t\}_{t \geq 1}$ is described as
\begin{align}\label{eq:etat}
\eta_t &= \frac{\ln 2}{\Delta_{t-1}(\eta_{1}^{t-1})} ,
\end{align}

where $\Delta_t(\eta_{1}^{t-1})$ is called the ``cumulative mixability gap" at time $t$ and is given by
\begin{align}\label{eq:deltat}
\Delta_t(\eta_{1}^{t-1}) &:= \sum_{s=1}^t \delta_s(\eta_s) \text{ where } \\
\delta_s(\eta_s) &:=  \inprod{\wvec_s(\eta_s)}{\lvec_s} + \frac{1}{\eta_s} \ln \inprod{\wvec_s(\eta_s)}{e^{-\eta_s \lvec_s}} .
\end{align}
\end{definition}

We are now ready to describe our main algorithm.
\begin{definition}
The algorithm \alg{ContextTreeAdaHedge$(D)$} whose prior is derived from the function $g(\cdot)$ updates its probability distribution on tree expert as follows:

\begin{align}\label{eq:naivemainupdate}
w^{(\mathsf{tree})}_{t,\boldf}(\eta_t;g) &= \frac{\left(\sum_{h=\text{order}(\boldf)}^D g(h)\right) e^{-\eta_t L_{t,\boldf}}}{\sum_{\boldf' \in \mathcal{F}_D} \left(\sum_{h=\text{order}(\boldf')}^D g(h)\right) e^{-\eta_t L_{t,\boldf'}}} .
\end{align}

and learning rate update $\{\eta_t\}_{t \geq 1}$ made according to Equations~\eqref{eq:etat} and~\eqref{eq:deltat}.

\end{definition}

The algorithm \alg{ContextTreeAdaHedge$(D)$} appears to have a prohibitive computational complexity of $\Oh(|\mathcal{F}_D|) = \Oh(2^{2^D})$.
However, the distributive law enables a clever reduction in computational complexity to $\Oh(2^D)$.
The main idea is that instead of keeping track of cumulative losses of all the $2^{2^D}$ functions in $\mathcal{F}_D$, represented by $\{L_{t,\boldf}\}_{\boldf \in \mathcal{F}_D}$, we only need to keep track of the cumulative losses of making certain predictions as a function of certain contexts, represented by $\{\{L_{x,t,y}\}_{y \in \Xspace}\}_{x \in \Xspace^D}$.
This reduction was first considered for tree expert prediction in the worst-case~\cite{helmbold1997predicting}, with a fixed learning rate $\eta > 0$, and can easily be extended to the broader class of exponential-weights updates. 
Proposition~\ref{prop:algequivalence}, which is stated and proved in Appendix~\ref{sec:algbenefits} for completeness, shows that the update on probability distribution on \textit{tree experts}, described in Equation~\eqref{eq:naivemainupdate}  -- can be equivalently written as a computationally faster update on probability distribution on \textit{predictors}:
\begin{subequations}
\begin{align}
w_{t,y}(\eta_t;g) &= \frac{\sum_{h=0}^D g'(h;\eta_t) e^{-\eta_t L_{X_t(h),t,y}}}{\sum_{h=0}^D g'(h;\eta_t) \left(\sum_{y \in \Xspace} e^{-\eta_t L_{X_t(h),t,y}}\right)} \text{ where } \label{eq:mainupdate} \\
g'(h;\eta_t) &= g(h) \prod_{x(h) \neq X_t(h)} \left(\sum_{y \in \Xspace} e^{-\eta_t L_{x(h),t,y}}\right)
\end{align}
\end{subequations}

The equivalence is in the sense that the expected loss incurred by updates~\eqref{eq:naivemainupdate} and~\eqref{eq:mainupdate} is the same.



\subsection{Potential generative assumptions on data}

As we have mentioned informally, we would like to get greatly improved regret rates for data generated in a certain way (without apriori knowledge of such generation).
We work with the following standard \textit{stochastic condition} on our data.
\begin{definition}
We say that our data $(X_t,Y_t)_{t \geq 1}$ satisfies the $d^{th}$-order stochastic condition if the following conditions hold:
\begin{enumerate}
\item The random vectors $\{(X_t,Y_t)\}_{t \geq 1}$ are independent and identically distributed across $t \geq 1$.
\item $X_t \text{ i.i.d. } \sim Q_d^*(\cdot)$, $Y_t | X_t \sim P^*(\cdot|X_t(d))$ for all $X_t \in \Xspace^D$.
\end{enumerate}

We denote the marginal distribution on $X_t(h)$ by $Q_h^*(\cdot)$.
For this setting, it is natural to define the best ``external predictor" for any $h \leq d$:
\begin{align}\label{eq:bestexternalpredictor}
f^*(x(h)) :\in {\arg \max}_{y \in \Xspace} P^*(y|x(h)) \text{ for all } x(h) \in \Xspace^h ,
\end{align}

For the special case of $h = d$, we assume that the best predictor is unique\footnote{This is the fundamental \textit{Tsybakov margin condition}~\cite{tsybakov2004optimal} that is essential for eventual learnability of the best predictor.}, i.e.
\begin{align*}
P^*(f^*(x(d))|x(d)) > P^*(y|x(d)) \text{ for all } y \neq f^*(x(d)) \text{ and for all } x(d) \in \Xspace^d .
\end{align*}

and denote the parameter
\begin{align}\label{eq:minmaxprob}
\beta(x(d)) &= P^*(f^*(x(d))|x(d)) \\
\betastar &:= \min_{x(d) \in \Xspace^d} \beta(x(d)) .
\end{align}
\end{definition}

Note that the uniqueness of best-predictor assumption directly implies that $\betastar > 1/2$, since we are working with a binary alphabet.

Based on this, we also define the important notions of asymptotic \textit{unpredictability} for all model orders $h \leq d$.
The definitions and notation are directly inspired by information-theoretic limits on sequence compression and prediction~\cite{feder1992universal}.
\begin{definition}[~\cite{feder1992universal}]
For data $(X_t,Y_t)_{t \geq 1}$ satisfying the $d^{th}$-order stochastic condition, we define its asymptotic unpredictability under the $h^{th}$-order predictive model by -- 
\begin{align}
\pi^*_h := \sum_{x(h) \in \Xspace^h } Q_h^*(x(h)) \left[1 - \max_{y \in \Xspace}\{P^*(y|x(h))\}\right]
\end{align}

For $h > d$, we have $\pi^*_h = \pi^*_d$.
For $h < d$, we have $\pi^*_h > \pi^*_d$.
\end{definition}

\section{Main results}\label{sec:results}

Different choices of the function $g(\cdot)$ used to describe the prior distribution on tree experts yield vastly different results.
Consider the choice $\gunif(h) := \Ind[h=D]$, which corresponds to the typical \textit{prior-free} implementation of exponential weights (i.e Equation~\eqref{eq:naivemainupdate} with a uniform prior).
With this choice, Proposition~\ref{prop:uniform} in Appendix~\ref{sec:trueorder} describes the ``best-of-both-worlds" bound that we obtain: worst-case regret $\Oh(\sqrt{T} \cdot 2^D)$, and regret $\Oh(2^{2D})$ in the stochastic case.
Note that the stochastic regret bound, while constant and thus independent of the horizon $T$, is highly suboptimal in its dependence on the maximum model order $D$.
The bound does not improve for drastically simpler cases; for example, $Y_t \sim \text{ i.i.d }$ and $Y_t$ is independent of $X_t$.


If we knew the true model order $d$, we would want to use \alg{ContextTreeAdaHedge$(d)$}.
We now show that a suitable choice of prior helps us effectively \textit{learn the model order}, as well as stay worst-case robust.
We study the algorithm with the following choice of \text{model-order-proportional} prior function.
\begin{align}\label{eq:preweightsprop}
g_{\mathsf{prop}}(h) &= 2^{-2^{h+1}}
\end{align}

\begin{theorem}\label{thm:contexttreeadahedge}
\begin{enumerate}
\item For any sequence $\{X_t,Y_t\}_{t =1}^T$ the algorithm \alg{ContextTreeAdaHedge$(D)$} with prior defined according to function $g_{\mathsf{prop}}(\cdot)$ gives us regret rate
\begin{align}\label{eq:worstcaserate}
R_{T,d} &= \Oh\left(\sqrt{T} 2^d \right) 
\end{align}

with respect to the best $d^{th}$-order tree expert in hindsight, and for every $d \in \{0,1,\ldots,D\}$.
\item Consider any $\delta \in (0,1]$.
Let the sequence $(X_t,Y_t)_{t \geq 1}$ satisfy the $d^{th}$-order stochastic condition with parameter $\betastar$.
Denote $\alpha_{d-1,d} := \frac{\pistar_{d-1} - \pistar_d}{2}$.
Then, \alg{ContextTreeAdaHedge$(D)$} with prior function
$g_{\mathsf{prop}}(\cdot)$ incurs regret with probability greater than
or equal to $(1 - \delta)$:
\begin{align}\label{eq:stochasticrate}
R_{T,d} &= \Oh\left(2^{2d} \left(\frac{d^2}{\alpha^2_{d-1,d}}\ln \left(\frac{d}{\alpha_{d-1,d}^2 \epsilon} \right) + \frac{D \cdot d }{(\alphastar)^2} \ln\left( \frac{D}{\alphastar \epsilon} \right)  \right)\right) 
\end{align}
where $\alphastar = \min\{\alpha_{d-1,d},2\betastar - 1)\}$.
\end{enumerate}
\end{theorem}

The proof of Theorem~\ref{thm:contexttreeadahedge} involves several moving parts to combine adversarial-stochastic interpolation and structural risk minimization, and we defer this proof to the appendix.
We provide an intuitive sketch of the proof in Section~\ref{sec:proofsketch}.

Theorem~\ref{thm:contexttreeadahedge} is the first result of its kind to obtain comparable regret rates as would be achieved by an algorithm that had oracle knowledge about the presence of stochasticity \textit{and} the model order.
This is the strongest possible side information that an algorithm could conceivably possess keeping the online learning problem non-trivial.
In simulation, we also demonstrate the significant empirical advantage of algorithms that achieve two-fold adaptivity over ``best-of-both-worlds" algorithms that \textit{do not adapt to model complexity}.
The advantage of offline data-driven model selection is well established, and we see this advantage even more naturally while measuring regret in online learning.


\section{Proof sketch of Theorem~\ref{thm:contexttreeadahedge}}\label{sec:proofsketch}

Initially, we mirror the established style of ``best-of-both-worlds" results.
The first step is always to prove a regret bound that is dependent on the data $\{(X_t,Y_t)\}_{t=1}^T$; in particular, a bound of the form $R_{T,d} = \Oh\left(\sqrt{V_T(\eta_1^T;\gprop)} \cdot 2^d\right)$ where $V_T(\eta_1^T;\gprop)$ represents the cumulative variance of loss incurred by the algorithm.
Curiously, we are easily able to get a bound (commonly called a second-order bound) that is adaptive to the model order using exponential weights with a prior\footnote{The careful reader will notice that there is nevertheless a suboptimality in the exponent as compared to the second-order bound obtained by algorithms like Squint~\cite{koolen2015second} and AdaNormalHedge~\cite{luo2015achieving}. 
However, the ``variance"-like terms in those results are different, as is their more complicated analysis for the iid case.
Until similar analysis is done for these algorithms, they are not immediately comparable.}!

The \textit{cumulative variance term} $V_T$ is telling us something about how random the randomized updates in the algorithm are.
In the worst case, $V_T \leq \frac{T}{4}$ and we automatically recover the adversarial result -- but often, this term can be significantly smaller.
It is easy to see that this randomness will greatly reduce when the losses are \textit{stochastic} in the sense that one tree expert looks consistently better than the others.
It will also reduce in the presence of a favorable prior $\gprop(\cdot)$ if that best expert possesses simpler structure.
However, all existing analysis~\cite{cesa2007improved,erven2011adaptive,de2014follow,luo2015achieving,koolen2015second,koolen2016combining} only exploits the former property, and not the latter -- thus giving a pessimistic scaling of $\Oh(2^D)$ for our problem.

Our main technical contribution is tackling the more difficult problem of finely controlling the cumulative variance under a favorable prior -- showing that it in fact scales as the significantly smaller $\sqrt{V_T} = \Oh(2^d)$.
We achieve this by making an explicit connection to \textit{probabilistic model selection by complexity regularization}.
To see this, consider Equation~\eqref{eq:naivemainupdate} written equivalently as the optimization problem in the Follow-the-Regularized Leader~\cite{shalev2012online} update:

\begin{align}\label{eq:naivemainupdatereg}
\wvec^{(\mathsf{tree})}_t &:= {\arg \min}_{\wvectree} \left[ \inprod{\wvectree}{\Lvectree_t} + \frac{1}{\eta_t} \left( - \underbrace{H(\wvectree)}_{\text{entropy regularization}} + \underbrace{\inprod{\wvectree}{\mathbf{C}^{(\mathsf{tree})}}}_{\text{complexity regularization}} \right)\right] ,
\end{align}

where $C^{(\mathsf{tree})}_{\boldf} := 2^{\text{order}(\boldf)} \log 2$ and $H(\cdot)$ denotes the entropy functional on a probability distribution over a discrete-valued random variable.
Viewed this way, the algorithm \alg{ContextTreeAdaHedge$(D)$} updates to minimize the cumulative loss \textit{adaptively} regularized with entropy (to protect against a potential adversary) and model complexity (to adapt to simpler models faster).

\quickfigure{fig:modelselection}{Illustration of the tradeoff between estimation error and approximation error for various choices of model order. 
The true model order is $4$ and the plot made is of performance of uniform-prior \alg{ContextTreeAdaHedge$(h)$} for different choices of $h$, measured at $T = 1500$.}{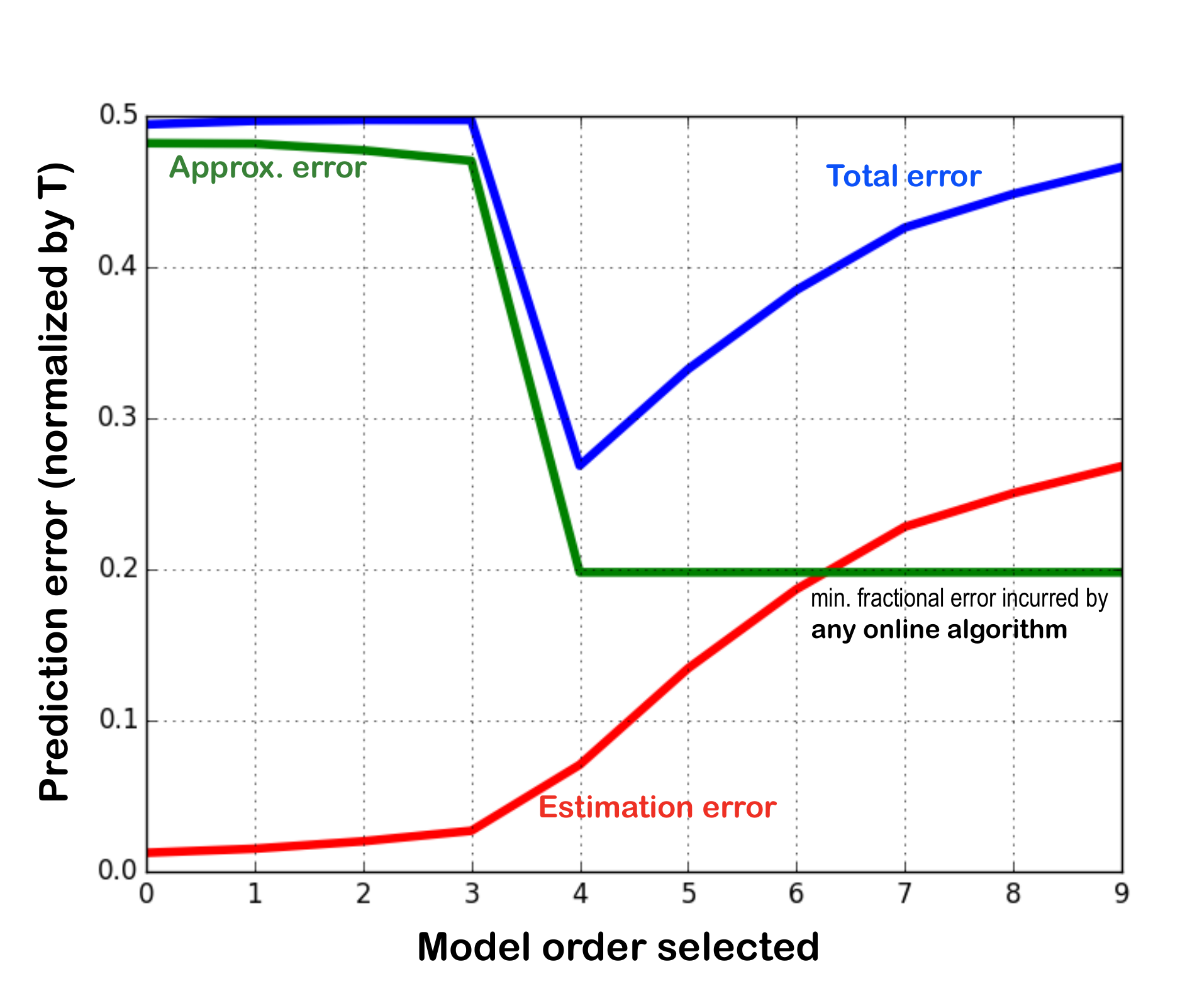}{0.33\textwidth}

Figure~\ref{fig:modelselection} illustrates the classical tradeoff in \textit{stochastic} model selection in an example where the true model order is $4$ -- the estimation error increases with model order, and the approximation error decreases with model order, and plateaus out at the true model order $4$ (note that this is the minimum average prediction error that any online learning algorithm should be expected to pay).
Clearly, the true model order minimizes the appropriate combination of estimation error and approximation error.
We show a \textit{probabilistic model selection guarantee}, i.e. we can pick the true model high probability.
We do this by ruling out lower and higher-order models alike.
On one hand, the more (superfluously) complex a model is, the more it is going to \textit{overfit}, contributing to unnecessary accumulated regret -- however, the more its unfavorable prior drags it down to rule it out.
On the other hand, the more (unnecessarily) simple a model is, the worse it is going to \textit{approximate} -- and since this approximation error is directly penalized in Equation~\eqref{eq:naivemainupdatereg}, the less likely it is to be picked.


The reason the classical analysis of stochastic model selection~\cite{massart2007concentration} does not directly apply here is in the requirement to adapt \textit{multi-fold}, between adversity and stochasticity of varying model complexity.
The primary technical difficulty is in characterizing the extent of adaptivity, encapsulated in the time-varying, \textit{data-dependent} learning rate which is known to be notoriously difficult to track~\cite{de2014follow,koolen2015second,koolen2016combining}.
It is perilous for the learning rate to remain too high (in which case the algorithm is effectively greedy, and overfits for too long), or sink too low (in which case we remain stuck selecting poorly fitting models).
Remarkably, we are able to carefully sandwich the learning rate in high probability to ensure model selection, \textit{in both cases} using the fundamental inverse relationship between the learning rate and regret that is used to \textit{learn the learning rate} in adaptive algorithms.
This clever relationship has been exploited to achieve stochastic-adversarial adaptivity; here, we show that its power is significantly higher, in being able to additionally adapt to model complexity\footnote{In fact, the same conceptual idea underlies the approaches to \textit{learn the learning rate}, prevalent in Squint, MetaGrad and AdaNormalHedge.}.
Once the (high-probability) model selection guarantee is obtained, analysis proceeds with slight generalization of the \alg{AdaHedge} analysis~\cite{erven2011adaptive} to the tree experts setting.

\section{Simulations}\label{sec:simulations}

We now provide a brief empirical illustration of the power of two-fold adaptivity to stochasticity \textit{and model complexity} with \alg{ContextTreeAdaHedge$(D)$} equipped with the prior function $\gprop(\cdot)$.

\multifigureexterior{fig:simulations}{Comparison of optimal greedy FTL, \alg{ContextTreeAdaHedge$(D)$} with uniform prior and prior function $\gprop(\cdot)$ (where $D = 8$); against a context-of-length-$3$ structure, upto $T = 1500$ rounds.}
{
\subfig{0.33\textwidth}{Total loss as a function of $T$.}{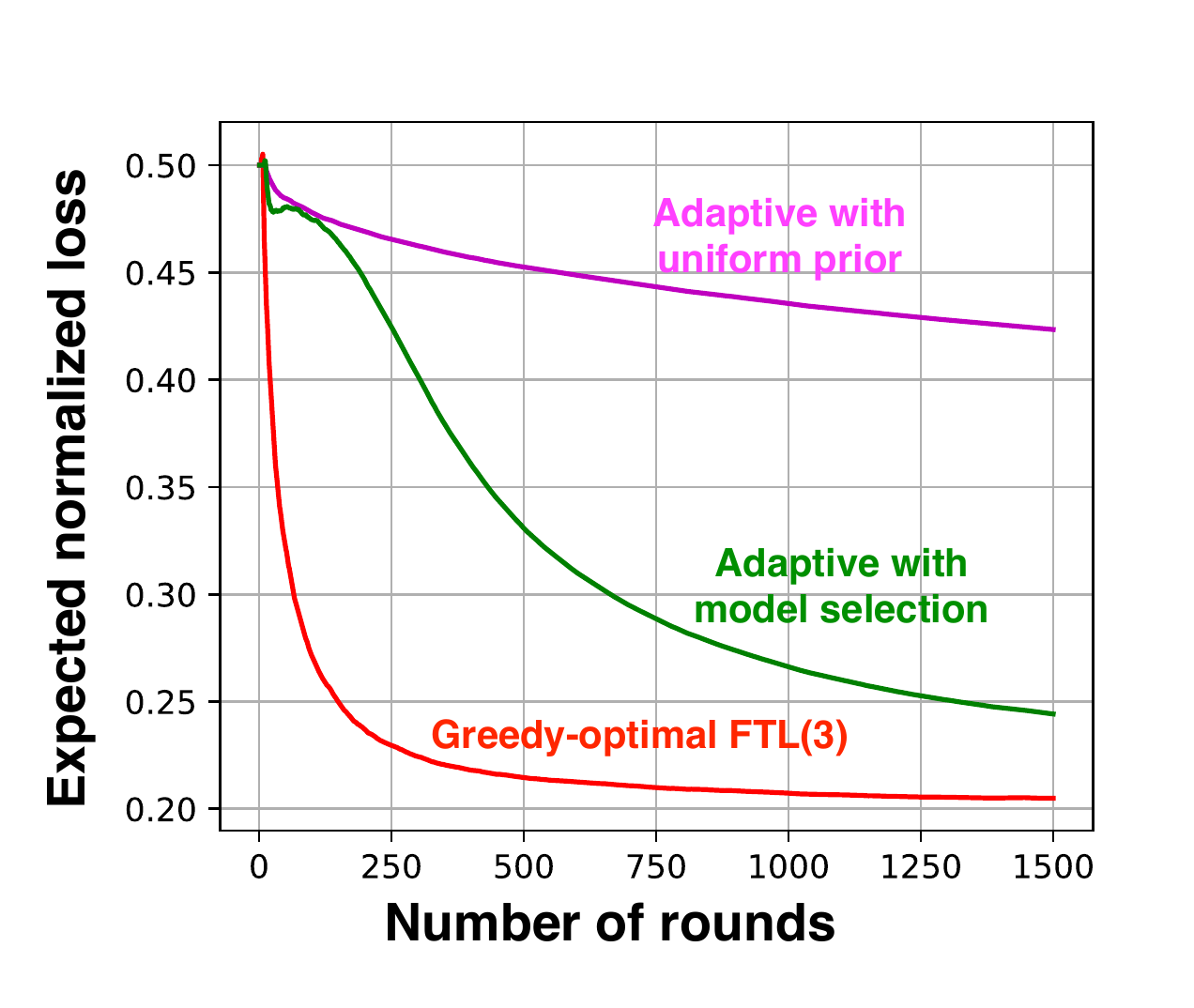}
\subfig{0.33\textwidth}{$R_{T,3}$ as a function of $T$.}{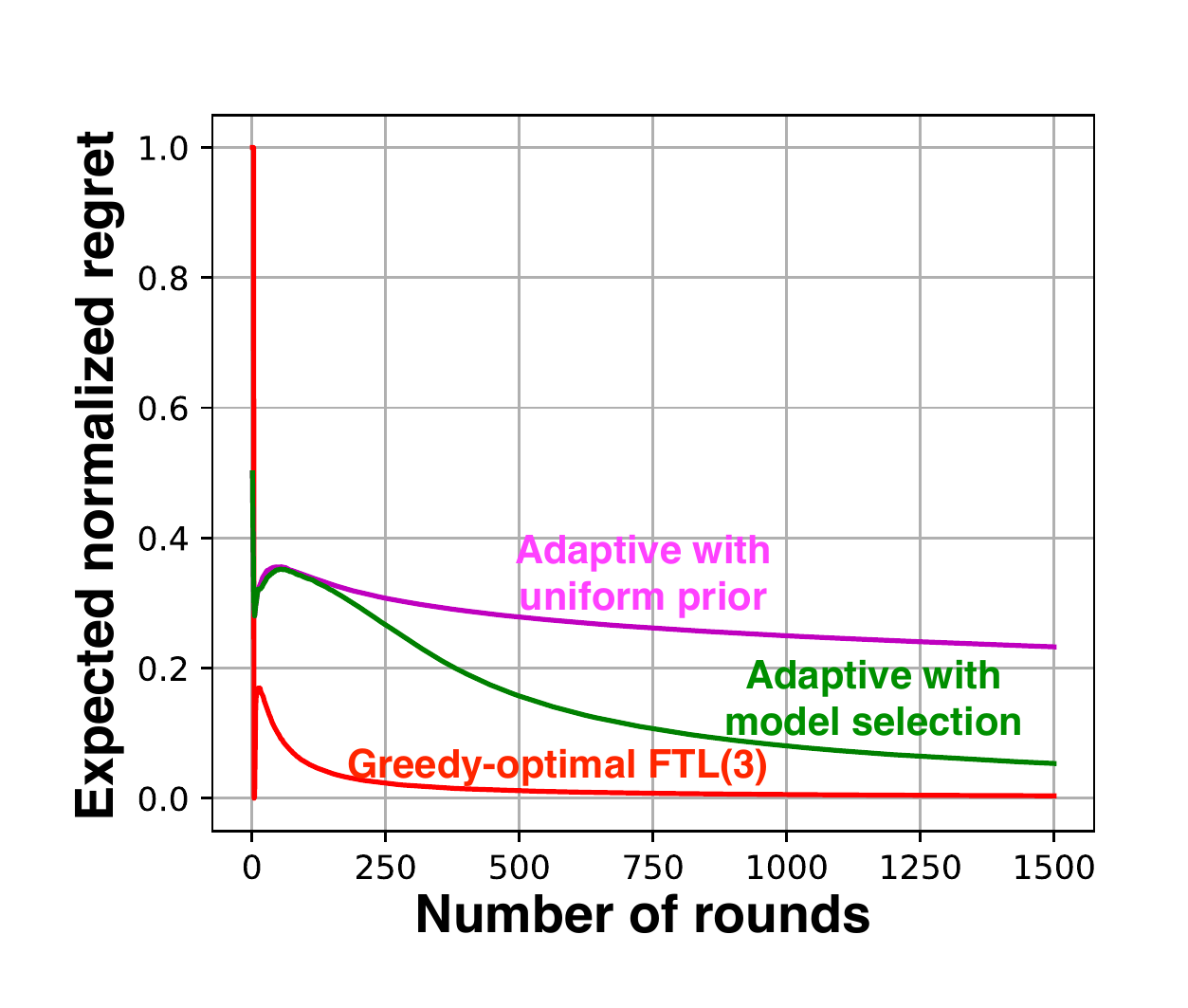}
}

We consider a $3^{rd}$-order-stochastic process such that $Y_t \sim \BER\left(0.6 \cdot \left(X_{t-3} \mathbin{\oplus} X_{t-2} \mathbin{\oplus} X_{t-1}\right) + 0.2 \right)$.
Figure~\ref{fig:simulations} compares three algorithms: the \textit{optimal online algorithm with oracle knowledge of this structure} (the greedy \alg{Follow-the-Context-Leader$(3)$}); uniform-prior \alg{ContextTreeAdaHedge$(D)$}, which adapts to stochasticity but not model order; and our two-fold adaptive algorithm, \alg{ContextTreeAdaHedge$(D)$} with the prior function $\gprop(\cdot)$.

Figure~\ref{fig:simulations} shows the expected normalized regret $\frac{R_{T,3}}{T}$ and expected normalized \textit{cumulative loss} of the algorithms.
We make two natural conclusions from Figure~\ref{fig:simulations}. 
One, that model adaptivity makes a tremendous difference to regret and overall loss: \alg{ContextTreeAdaHedge$(D)$} equipped with uniform prior does not adapt to model order, and pays for it with loss (regret) accumulated due to overfitting.
Two, that our main adaptive algorithm, which is effectively \textit{learning the presence of stochasticity and the right model order} is remarkably competitive with the optimal Follow-the-Leader algorithm, which possesses oracle knowledge of both.
Viewed another way, this competitiveness of adaptive algorithms suggests that there is only a small price to pay to incorporate adversarial robustness in existing stochastic model selection frameworks.
Appendix~\ref{sec:extra} provides an additional example of an iid process on $\{Y_t\}_{t \geq 1}$, the simplest possible model, which further illuminates both the positives of adaptivity and the negatives of lack of adaptivity.

\section{Discussion}

\paragraph{Summarization of contributions}

We study the problem of binary contextual prediction (easily generalizable to $m$-ary contextual prediction) with $0-1$ loss.
We design an algorithm that incorporates recent advances in adaptivity with contextual pre-weighting, and show that we can simultaneously adapt to the model order complexity \textit{and} the existence of stochasticity.
By adaptively recovering the stochastic structural risk minimization framework, we are able to select the right $d^{th}$-order model for the stochastic process, and obtain regret rates that are competitive with those of the optimal greedy algorithm which knows not only the presence of stochastic structure, but the exact value of $d$.
As far as we know, our work provides the first perspective on online \textit{stochastic} model selection in a more challenging environment where we need to distinguish between actual stochasticity and adversity: the case where the data is not, in fact, coming from any of these models.

\paragraph{Future directions}

Many future directions arise from this work.
First, we acknowledge that the regret rate we obtain is not exactly optimal, particularly in terms of the multiplicative factor of $d$ in the exponent.
It would be interesting to understand whether we can further improve on this factor in our bound, either by analyzing other existing algorithms that learn the learning rate~\cite{koolen2015second,luo2015achieving}, or devising a new approach altogether.
The simpler experts setting was the first natural choice to study this question, and we are hopeful that the positive results obtained here can be generalized to online optimization to develop a universal theory for simultaneous model selection and stochastic adaptivity.
Recent advances have been made, separately, in both of these areas~\cite{orabona2014simultaneous,van2016metagrad,orabona2016coin,foster2017parameter}).
We are also interested in studying these problems for limited-information feedback, which would lead to the contextual bandits setting.

\subsubsection*{Acknowledgments}

We would like to thank Sebastien Gerchinovitz for useful discussions.
We gratefully acknowledge the support of the NSF through grants AST-1444078, ECCS-1343398, CNS-1321155 and IIS-1619362.
We also credit the DARPA Spectrum Challenge for inspiring some of the ideas in this work, and generous gifts from Futurewei.

\bibliographystyle{alpha}
\bibliography{OLreferences.bib}


\appendix

\section{Main proofs of \alg{ContextTreeAdaHedge$(D)$}}

\subsection{Second-order regret bound and adversarial result}

\begin{table}
\centering
\begin{tabular}{ |c|c| } 
\textbf{Notation} & \textbf{Meaning} \\
\hline
 \hline
 $\mathcal{F}_h$ & Set of all $h^{th}$-order tree experts $\boldf: \Xspace^h \to \Xspace$  \\ 
$\ltree_{t,\boldf} \Big{|} \Ltree_{t,\boldf}$ & Instantaneous | cumulative loss suffered by tree expert $\boldf$ \\ 
$\lvectree_{t,\boldf} = [\ltree_{t,\boldf}]_{\boldf \in \mathcal{F}_D} \Big{|} \Lvectree_{t,\boldf} = [\Ltree_{t,\boldf}]_{\boldf \in \mathcal{F}_D}$ & Vector of instantaneous | cumulative losses suffered by tree experts in $\mathcal{F}_D$ \\ 
$l_{t,y}$ & Instantaneous loss at time $t$ suffered by predicting $y \in \Xspace$  \\
$L_{x,t,y}$ & Cumulative loss obtained by predicting $y \in \Xspace$ after seeing $x \in \Xspace^h$ \\
$\widehat{F}_h(t) := {\arg \min}_{\boldf \in \mathcal{F}_h} L_{t,\boldf}$ & Best $h^{th}$-order tree expert at time $t$ \\
$\widehat{L}_{t,h} := L_{t,\widehat{F}_h(t)}$ & Cumulative loss suffered by tree expert $\widehat{F}_h(t)$ \\
$R_{T,h}$ & Regret suffered with respect to best $h^{th}$-order tree expert \\
\hline
\hline
\end{tabular}
\caption{Basic notation for regret minimization under contextual experts framework.}\label{tab:lossdata}
\end{table}

\begin{table}
\centering
\begin{tabular}{ |c|c| } 
\textbf{Notation} & \textbf{Meaning} \\
\hline
 \hline
$\eta_1^T = \{\eta_t\}_{t=1}^T$ & Sequence of learning-rates used in exponential weights updates  \\ 
 $g: \{0,1,\ldots,D\} \to \reals_{+}$ & Function for prior on tree experts function of order. \\ 
 $\wvec_1(g) \Big{|} \wvec_1^{(\mathsf{tree})}(g)$ & Initial distribution on prediction | choice of tree expert  \\ 
 $\wvec_t(\eta_t;g) \Big{|} \wvec_t^{(\mathsf{tree})}(\eta_t;g)$ & Distribution at round $t$ on prediction | choice of tree expert  \\ 
 $Z(g)$ & Normalizing factor for initial distribution on tree experts \\
 $h_t(\eta_t;g) \Big{|} H_t(\eta_1^t;g)$ & Instantaneous | cumulative expected loss incurred by algorithm at time $t$ \\
 $\delta_t(\eta_t;g) \Big{|} \Delta_t(\eta_1^t;g)$ & Instantaneous | cumulative mixability gap of algorithm at time $t$ \\
 $v_t(\eta_t;g) \Big{|} V_t(\eta_1^t;g)$ & Instantaneous | cumulative variance of loss incurred by algorithm at time $t$ \\
 \hline
 \hline
\end{tabular}
\caption{Notation specific to algorithm \alg{ContextTreeAdaHedge}.}\label{tab:alg}
\end{table}

We first obtain our second-order-regret bound, stated generally for a prior function $g: \{0,1,\ldots,D\} \to \reals$.
Tables~\ref{tab:lossdata} and~\ref{tab:alg} recap the basic notation for regret minimization and important algorithmic notation, and are useful to look at while reading the proof of the second-order bound.

Recall the expression for the computationally naive update in Equation~\eqref{eq:naivemainupdate}:
\begin{align*}
w^{(\mathsf{tree})}_{t,\boldf}(\eta_t;g) &= \frac{\left(\sum_{h=\text{order}(\boldf)}^D g(h)\right) e^{-\eta_t L_{t,\boldf}}}{\sum_{\boldf \in \mathcal{F}_D} \left(\sum_{h=\text{order}(\boldf)}^D g(h)\right) e^{-\eta_t L_{t,\boldf}}} .
\end{align*}

and the expression for the initial distribution on tree experts based on Definition~\ref{def:preweights}:
\begin{align*}
w^{(\mathsf{tree})}_{1,\boldf}(g) = \frac{\sum_{h=\text{order}(f)}^D g(h)}{Z(g)}
\end{align*}

where $Z(g) > 0$ is the initial normalizing factor.
The explicit expression for the normalizing factor is $Z(g) = \sum_{h=0}^D 2^{2^h} g(h)$.

\begin{lemma}\label{lem:secondorderbound}
\alg{ContextTreeAdaHedge$(D)$} with prior function $g(\cdot)$ obtains regret
\begin{align*}
R_{T,d} \leq \left(\sqrt{V_T \ln 2} + \frac{2}{3} \ln 2 + 1\right)\left(1 + \frac{\ln \left(\frac{Z(g)}{g(d)}\right)}{\ln 2} \right)
\end{align*}

for every $d \in \{0,1,\ldots,D\}$.
\end{lemma}

\begin{proof}
Recall that $\widehat{F}_d(T)$ denotes the best $d^{th}$-order tree expert at round $T$ for the given loss sequence.
We denote $\widehat{L}_{T,d} := L_{t,\widehat{F}_d(t)}$ as the actual loss incurred by this expert.
We start with the computationally naive update in probability distribution over tree experts as in Equation~\eqref{eq:naivemainupdate}, and the proof proceeds in a very similar manner to the variance-based regret bound for vanilla AdaHedge~\cite{de2014follow}.
We denote 
\begin{align*}
h_t(\eta_t;g) &:= \inprod{\wvec_t(\eta_t;g)}{\lvec_t} = \inprod{\wvec^{(\mathsf{tree})}_t(\eta_t;g)}{\lvec^{(\mathsf{tree})}_t}\\
H_T(\eta_1^T;g) &:= \sum_{t=1}^T h_t(\eta_t;g) \\
m_t(\eta_t;g) &:= \frac{1}{\eta_t} \ln \inprod{\wvec_t(\eta_t;g)}{e^{-\eta_t \lvec_t}} = \frac{1}{\eta_t} \ln \inprod{\wvec_t^{(\mathsf{tree})}(\eta_t;g)}{e^{-\eta_t \lvec_t^{(\mathsf{tree})}}}  \\
M_T(\eta_1^T;g) &:= \sum_{t=1}^T m_t(\eta_t;g) .
\end{align*}

Recall that the mixability gap $\delta_t(\eta_t;g) = h_t(\eta_t;g) - m_t(\eta_t;g)$ and $\Delta_T(\eta_1^T;g) = \sum_{t=1}^T \delta_t(\eta_t;g)$.
Since the instantaneous losses are bounded between $0$ and $1$, it is easy to show that $0 \leq \delta_t(\eta_t;g) \leq 1$.

A standard argument tells us that 
\begin{align*}
R_{T,d} &= H_T(\eta_1^T;g) - L^*_{T,d} \\
&= H_T(\eta_1^T;g) - M_T(\eta_1^T;g) + M_T(\eta_1^T;g) - L^*_{T,d} \\
&= M_T(\eta_1^T;g) - L^*_{T,d} + \Delta_T(\eta_1^T;g) .
\end{align*}

Recall that the sequence $\eta_1^T$ is decreasing as an automatic consequence of the update in Equation~\eqref{eq:etat}, and non-negativity of $\delta_t$.
Handling a time-varying, data-dependent learning rate is well known to be challenging~\cite{erven2011adaptive,de2014follow}.
We invoke a simple lemma from the original proof of AdaHedge~\cite{de2014follow} that helps us effectively subsitute the final learning rate.
\begin{lemma}[\cite{de2014follow}]\label{lem:decetat}
For any exponential-weights update with a decreasing learning rate $\eta_1^T$ and prior function $g(\cdot)$, we have $M_T(\eta_1^T;g) \leq M_T(\{\eta_T\}_{t=1}^T;g)$.
\end{lemma}

Thus, we get
\begin{align}\label{eq:intermediate}
R_{T,d} \leq M_T(\{\eta_T\}_{t = 1}^T;g) - L^*_{T,d} + \Delta_T(\eta_1^T;g) .
\end{align}

We also have the following simple intermediate result for $M_T(\{\eta_T\}_{t=1}^T;g)$, which is simply a slightly more general version of the lemma in~\cite{de2014follow} that can apply to non-uniform priors.
\begin{lemma}\label{lem:mtlemma}
\begin{align*}
M_T(\{\eta_T\}_{t=1}^T;g) \leq L^*_{T,d} + \frac{1}{\eta_T} \ln \left(\frac{Z(g)}{g(d)}\right) .
\end{align*}
\end{lemma}

\begin{proof}
We note that 
\begin{align*}
\inprod{\wvec^{(\mathsf{tree})}_1(g)}{e^{-\eta_T \mathbf{L}^{(\mathsf{tree})}_T}} \geq w^{(\mathsf{tree})}_{1,f^*_{T,d}}(g) e^{-\eta_T L^*_{T,d}} .
\end{align*}

Because the initial distribution $\wvectree_1$ is normalized to sum to $1$, a simple telescoping argument can be used to give $M_T(\{\eta_T\}_{t=1}^T;g) = \sum_{t=1}^T m_t(\{\eta_T\}_{t=1}^T;g) = - \frac{1}{\eta_T} \ln \left( \inprod{\wvec^{(\mathsf{tree})}_1(g)}{e^{-\eta_T \mathbf{L}^{(\mathsf{tree})}_T}} \right)$.

This automatically tells us that
\begin{align*}
M_T(\{\eta_T\}_{t=1}^T;g) &= - \frac{1}{\eta_T} \ln \left( \inprod{\wvec^{(\mathsf{tree})}_1(g)}{e^{-\eta_T \mathbf{L}^{(\mathsf{tree})}_T}} \right) \\
&\leq - \frac{1}{\eta_T} \ln (w^{(\mathsf{tree})}_{1,f^*_{T,d}}(g)) + L^*_{T,d} \\
&= L^*_{T,d} + \frac{1}{\eta_T} \ln \left(\frac{1}{w^{(\mathsf{tree})}_{1,f^*_{T,d}}(g)}\right) \\
&= L^*_{T,d} + \frac{1}{\eta_T} \ln \left(\frac{Z(g)}{\sum_{h = d}^D g(h)}\right) \\
&\leq L^*_{T,d} + \frac{1}{\eta_T} \ln \left(\frac{Z(g)}{g(d)}\right)
\end{align*}
thus proving the lemma.
\end{proof}

Now, Equation~\eqref{eq:intermediate} and Lemma~\ref{lem:mtlemma} together with the definition of $\eta_t$ in Equation~\eqref{eq:etat} give us
\begin{align*}
R_{T,d} &\leq \frac{1}{\eta_T} \ln \left(\frac{Z(g)}{g(d)}\right) + \Delta_T(\eta_1^T;g) \\
&=  \frac{\ln \left(\frac{Z(g)}{g(d)}\right)}{\ln 2} \Delta_{T-1}(\eta_1^{T-1};g) + \Delta_T(\eta_1^T;g) .
\end{align*}

From non-negativity of $\delta_t$, we have $\Delta_{T-1}(\eta_1^T;g) \leq \Delta_T(\eta_1^T;g)$ and so
\begin{align}\label{eq:rtdeltat}
R_{T,d} \leq \Delta_T(\eta_1^T;g) ( 1 + \frac{\ln \left(\frac{Z(g)}{g(d)}\right)}{\ln 2} ) .
\end{align}

It now remains to bound the quantity $\Delta_T$ in terms of variance.
In fact, it will be useful to define slightly more generic quantities
\begin{align*}
\Delta_{T_0}^T(\eta_{T_0}^T;g) &:= \sum_{t = T_0}^T \delta_t(\eta_t;g) \\
V_{T_0}^T(\eta_{T_0}^T;g) &:= \sum_{t=T_0}^T v_t(\eta_t;g) \text{ where } \\
v_t(\eta_t;g) &:= \var_{K_t \sim \wvec_t(\eta_t;g)}\left[l_{t,K_t}\right] .
\end{align*}

The bound is described below.
\begin{lemma}\label{lem:deltatvt}
We have
\begin{align*}
\Delta_{T_0}^T(\eta_{T_0}^T;g) \leq \sqrt{V_{T_0}^T(\eta_{T_0}^T;g) \ln 2} + \left(\frac{2}{3} \ln 2 + 1\right) .
\end{align*}
\end{lemma}

\begin{proof}
The argument is similar to the original AdaHedge proof~\cite{de2014follow} and proceeds below.
We use a telescoping sum to get
\begin{align*}
\Big(\Delta_{T_0}^T\big(\eta_{T_0}^T;g\big)\Big)^2 &= \sum_{t=T_0 + 1}^T \Big(\Delta_{T_0}^t\big(\eta_{T_0}^t;g\big)\Big)^2 - \Big(\Delta_{T_0}^{t-1}\big(\eta_{T_0}^{t-1};g\big)\Big)^2 \\
&= \sum_{t=T_0}^T \Big(\Delta_{T_0}^{t-1}\big(\eta_{T_0}^{t-1};g\big) + \delta_t\big(\eta_t;g\big)\Big)^2 - \Big(\Delta_{T_0}^{t-1}\big(\eta_{T_0}^{t-1};g\big)\Big)^2  \\
&= \sum_{t=T_0}^T 2\delta_t\big(\eta_t;g\big) \Delta_{T_0}^{t-1}\big(\eta_{T_0}^{t-1};g) + \Big(\delta_t\big(\eta_t;g\big)\Big)^2 \\
&\leq \sum_{t=T_0}^T 2\delta_t\big(\eta_t;g\big) \Delta_{t-1}\big(\eta_1^{t-1};g\big) + \Big(\delta_t\big(\eta_t;g\big)\Big)^2 \\
&= \sum_{t=T_0}^T 2 \delta_t\big(\eta_t;g\big) \frac{\ln 2}{\eta_t} + \Big(\delta_t\big(\eta_t;g\big)\Big)^2) \\
&\leq \sum_{t=T_0}^T  2 \delta_t\big(\eta_t;g\big) \frac{\ln 2}{\eta_t} + \delta_t\big(\eta_t;g\big) \text{ since $\delta_t(\eta_t;g) \leq 1$ } \\
&\leq (2 \ln 2 ) \sum_{t=T_0}^T \frac{\delta_t(\eta_t;g)}{\eta_t} + \Delta_{T_0}^T\big(\eta_{T_0}^T;g\big) .
\end{align*}

We also recall the following lemma from the original proof of AdaHedge~\cite{de2014follow}. 
The proof of this lemma involves a Bernstein tail bounding argument.
\begin{lemma}[\cite{de2014follow}]\label{lem:bernsteinbound}
We have
\begin{align*}
\frac{\delta_t\big(\eta_t;g\big)}{\eta_t} \leq \frac{1}{2} v_t\big(\eta_t;g\big) + \frac{1}{3} \delta_t\big(\eta_t;g\big) .
\end{align*}
\end{lemma}

Using Lemma~\ref{lem:bernsteinbound}, we then get
\begin{align}\label{eq:deltatquadratic}
\Big(\Delta_{T_0}^T\big(\eta_{T_0}^T;g\big)\Big)^2 \leq V_{T_0}^T\big(\eta_{T_0}^T;g\big) \ln 2 + \left(\frac{2}{3}\ln 2 + 1\right) \Delta_{T_0}^T\big(\eta_{T_0}^T;g\big)
\end{align}

which is an inequality for the quantity $\Delta_{T_0}^T\big(\eta_{T_0}^T;g\big)$ in quadratic form.
We now solve Equation~\eqref{eq:deltatquadratic}, and use Fact~\ref{f:quadraticformula} from Appendix~\ref{sec:algebra} to get 
\begin{align}\label{eq:deltat}
\Delta_{T_0}^T\big(\eta_{T_0}^T;g\big) \leq \sqrt{V_{T_0}^T\big(\eta_{T_0}^T;g\big) \ln 2} + \frac{2}{3} \ln 2 + 1 .
\end{align}

\end{proof}

Now we complete the proof of Lemma~\ref{lem:secondorderbound} by combining Equations~\eqref{eq:rtdeltat} and~\eqref{eq:deltat} for the special case of $T_0 = 1$.
\end{proof}


Now, noting that $V_T(\eta_1^T;g) \leq \frac{T}{4}$ and substituting the expression for $g = \gprop$ from Equation~\eqref{eq:preweightsprop} directly proves Equation~\eqref{eq:worstcaserate} from Lemma~\ref{lem:secondorderbound}.
To see this, we substitute $g = \gprop$ into the statement of Lemma~\ref{lem:secondorderbound} to get 
\begin{align*}
R_{T,d} &\leq \left(\sqrt{V_T(\eta_1^T;g) \ln 2} + \frac{2}{3} \ln 2 + 1\right)\left(1 + \frac{\ln \left(\frac{Z(\gprop)}{\gprop(d)}\right)}{\ln 2} \right) \\
&= \left(\sqrt{V_T(\eta_1^T;g) \ln 2} + \frac{2}{3} \ln 2 + 1\right)\left(1 + \frac{\ln \left(\frac{\sum_{h=0}^D 2^{2^h} 2^{-2^{h+1}}}{2^{-2^{d+1}}}\right)}{\ln 2} \right) \\
&= \left(\sqrt{V_T(\eta_1^T;g) \ln 2} + \frac{2}{3} \ln 2 + 1\right)\left(1 + \frac{\ln \left(\frac{\sum_{h=0}^D 2^{-2^h}}{2^{-2^{d+1}}}\right)}{\ln 2} \right) \\
&\leq \left(\sqrt{V_T(\eta_1^T;g) \ln 2} + \frac{2}{3} \ln 2 + 1\right)\left(1 + \frac{\ln \left(2 \cdot 2^{2^{d+1}}\right)}{\ln 2} \right) \\
&= \left(\sqrt{V_T(\eta_1^T;g) \ln 2} + \frac{2}{3} \ln 2 + 1\right)\left(2 + 2^{d+1}\right) \\
&\leq \left(\frac{1}{2} \sqrt{T \ln 2} + \frac{2}{3} \ln 2 + 1\right)\left(2 + 2^{d+1}\right)
\end{align*}

which is precisely Equation~\eqref{eq:worstcaserate} when expressed in big-$\Oh$ notation.

\subsection{Exploiting stochasticity}

To effectively bound regret for the ``easier" stochastic instances, we need finer control on the cumulative mixability gap term $\Delta_T(\eta_1^T;g)$.
Our starting point is the following thresholding lemma.
\begin{lemma}\label{lem:thresholding}
Fix $t_0 > 0$. Let $T_0 := \max\{0 < t \leq T: \eta_t > \frac{\ln 2}{t_0} \}$.
Then, we have
\begin{align}\label{eq:thresholding}
\Delta_T(\eta_1^T;g) &\leq t_0 + 1 + \sqrt{V_{T_0}^T(\eta_{T_0}^T;g) \ln 2} + \frac{2}{3} \ln 2 + 1 .
\end{align}
\end{lemma}

\begin{proof}
From the definition of $T_0$, we observe that
\begin{align*}
\eta_{T_0} &= \frac{\ln 2}{\Delta_{T_0-1}(\eta_1^{T_0 - 1};g)} > \frac{\ln 2}{t_0} \\
\implies \Delta_{T_0-1}(\eta_1^{T_0-1};g) &< t_0 \\
\implies \Delta_{T_0}(\eta_1^{T_0};g) &< t_0 + 1 .
\end{align*}

Then, using $\Delta_T(\eta_1^T;g) = \Delta_{T_0}(\eta_1^{T_0};g) + \Delta_{T_0}^T(\eta_{T_0}^T;g)$ and Lemma~\ref{lem:deltatvt} directly gives us the statement in Equation~\eqref{eq:thresholding} and completes the proof.
\end{proof}

We observe that the threshold $T_0$ depends on the choice of $t_0$ as well as the data (in fact, it is a random variable when the process $\{(X_t,Y_t)\}_{t=1}^T$ is stochastic).
We have the freedom to choose $t_0 > 0$ for our analysis.
Conceptually, in the stochastic regime, the choice of $t_0$ thresholds the number of rounds $T_0$ below which we can make few, if any, statistical guarantees, and will become clear in subsequent sections.
Effectively, Lemma~\ref{lem:thresholding} uses the elegant inverse relationship between learning rate and mixability (in Equation~\eqref{eq:etat}) to show that a minimal amount of regret, precisely, in terms of $t_0$, is accumulated \textit{even before we can make high-probability statistical guarantees}.

\subsubsection{Notation for contextual prediction}

\begin{table}\label{tab:analysis}
\centering
\begin{tabular}{ |c|c| } 
\textbf{Notation} & \textbf{Meaning/Interpretation} \\
\hline
 \hline
$N_t(x(h))$ & Appearance frequency of a sub-context $x(h) \in \Xspace^h$ \\ 
$\Phat_t(h|x(h))$ & Fraction of times that we observed $X_t(h) = x(h),Y_t = y$  \\
$S_{t,h}$ & Number-of-seen sub-contexts of length $h$ at time $t$ \\
$\pihat_h(t)$ & Estimated unpredictability based on $h^{th}$-order tree expert predictors \\
$D_t(h)$ & Gap between correct and incorrect predictors at time $t$ \\
$\wvec_t^{(h)}$ & Probability distribution on predictions \\
$v_t^{(h)}$ & Variance of loss of \alg{ContextTreeAdaHedge$(h)$} with uniform prior at time $t$ \\
$q_t(h) \propto Q_t(h)$ & Posterior probability that the $h^{th}$-order model is the right model \\
$d$ & True model order of data $(X_t,Y_t)_{t=1}^T$ \\
$Q^*_h(\cdot), h \leq d$ & Marginal distribution on $X_t(h), h \leq d$ \\
$P^*(\cdot|x(h))$ & Conditional distribution on $Y_t$ given $X_t = x(h)$ \\
$\beta(x(d)), \betastar$ & Average prediction accuracy with conteext $x(d)$ \\
$\pistar_h, h \leq D$ & Asymptotic unpredictability under $h^{th}$-order model.\\
$\thigh(h), h > d$ & Number of epochs of $x(h) \in \Xspace^h$ after which we can guarantee a unique best predictor \\
$\tlow(h), h \leq d$ & Number of rounds after which we can conclusively rule out lower $h^{th}$-order model \\
\hline
\hline
\end{tabular}
\caption{Notation for analysis.}
\end{table}

First, we define a couple of convenient counts for the number of appearances of a particular context, and the number of contexts that have so far appeared.
\begin{definition}
The \textbf{appearance frequency} of a particular context $x(h) \in \Xspace^h$ at time $t$ is given by
\begin{align*}
N_t(x(h)) := \sum_{s=1}^{t-1} \Ind[X_s(h) = x(h)] ,
\end{align*}

The \textbf{fraction of times the value $y \in \Xspace$} seen after a particular context is given by
\begin{align*}
\Phat_t(y|x(h)) &:= \frac{\sum_{s=1}^{t-1} \Ind[X_s(h) = x(h),Y_s = y]}{\sum_{s=h}^{t-1} \Ind[X_s(h) = x(h)]} \\
\Big(&= 1 - \frac{L_{x(h),t-1,y}}{N_t(x(h))}\Big)
\end{align*}

The \textbf{number-of-seen-contexts} is given by
\begin{align*}
S_{t,h} := \sum_{x(h) \in \Xspace^h} \Ind[N_t(x(h)) > 0] .
\end{align*}
\end{definition}

Next, we define our estimates for unpredictability, effectively an estimate for the approximation error, under various model orders.
\begin{definition}[\cite{feder1992universal}]\label{def:pihat}
For every value of $h \geq 0$ and a sequence $\{(X_t,Y_t)\}_{t \geq 1}$, we define its \textit{estimated unpredictability}
\begin{align*}
\pihat_h(t) &:= \sum_{x(h) \in \Xspace^h} \frac{N_t(x(h))}{t} \left(1 - \max_{y \in \Xspace}\{\Phat_t(y|x(h)) \}\right)\\
&= \sum_{x(h) \in \Xspace^h} \frac{1}{t} \min_{y \in \Xspace} \{L_{x(h),t,y}\} .
\end{align*}
\end{definition}

This definition is inspired by the information-theoretic perspective on universal sequence prediction~\cite{feder1992universal}.
In this line of work, the quantity $\pihat_h(t)$ represents the estimated unpredictability of a binary sequence under a $h$-memory Markov model.
This is the natural estimate of \textit{approximation error of the $h^{th}$-order model} that is used to carry out data-driven model selection.

Finally, we denote the \textit{true prediction} (the one we would make if we had oracle knowledge of the best predictor $f^*(\cdot)$) as 
\begin{align*}
Y^*_t := f^*(X_t(d)) .
\end{align*}

Then, for every $h \geq d$ we define
\begin{align}\label{eq:dtcontexts}
D_t(h) := L_{X_t(h),t,1 - Y^*_t} - L_{X_t(h),t,Y^*_t} 
\end{align}

represents the ``gap" between the correct predictor $Y^*_t$ and the worse predictor $1 - Y^*_t$ at time $t$, and pertaining to the current context $X_t(h)$.

\subsubsection{Explicit model selection}

We have stated the problem of wanting to exploit the structure of a $d^{th}$-order stochastic sequence $\{(X_t,Y_t)\}_{t \geq 1}$ in an online fashion, as a model selection problem.
This has been implicitly clear in the choice of prior function in Equation~\eqref{eq:preweightsprop}: more complex experts are downweighted.
Now, we make the connection clear.

As a reminder, we evaluate the performance of the algorithm \alg{ContextTreeAdaHedge$(D)$} with prior function $g_{\mathsf{prop}}(\cdot))$, and using Equation~\eqref{eq:deltat} as a jumping point, we are concerned with bounding the cumulative variance $V_{T_0}^T(\eta_{T_0}^T;g)$.

First, we observe that
\begin{align*}
V_{T_0}^T(\eta_{T_0}^T;\gprop) &= \sum_{t=T_0}^T v_t(\eta_t;\gprop) \\
&= \sum_{t=T_0}^T w_{t,Y_t^*}(\eta_t;\gprop) \left(1 - w_{t,1-Y_t^*}(\eta_t;\gprop)\right) \text { since $l_{t,K_t} \text{ i.i.d } \sim \BER(w_{t,1})$ } \\
&\leq \sum_{t=T_0}^T w_{t,1-Y_t^*}(\eta_t;\gprop)
\end{align*}

and thus, it is sufficient to control the evolution of the term $w_{t,1-Y_t^*}(\eta_t;\gprop)$ with $t$.
This is the probability with which we select the prediction $1 - Y_t^*$ that is more likely to be wrong under the stochastic model for the data.

The first step is to express the update in this probability in terms of a posterior probability on the effective \textit{order of the model} the algorithm is selecting.
Explicitly, we can re-write Equation~\eqref{eq:mainupdate} as 
\begin{align*}
w_{t,1-Y_t^*}(\eta_t;\gprop) &= \sum_{h=0}^D q_t(h;\eta_t,\gprop) w_{t,1-Y_t^*}^{(h)}(\eta_t)
\end{align*}

where we have defined the shorthand notation for the update used by \alg{ContextTreeAdaHedge$(h)$} with uniform prior,
\begin{align*}
w^{(h)}_{t,1 - Y_t^*}(\eta_t) := w_{t,1-Y_t^*}(\eta_t;\gunif) = \frac{e^{-\eta_t D_t(h)}}{1 + e^{-\eta_t D_t(h)}} ,
\end{align*}

where $D_t(h)$ is according to Equation~\eqref{eq:dtcontexts} and the quantities $\{q_t(h;\eta_t,\gprop)\}$ are explicitly written as 
\begin{align}\label{eq:qthactual}
q_t(h;\eta_t,\gprop) &\propto Q_t(h;\eta_t,\gprop) := \gprop(h) \prod_{x(h) \in \Xspace^h} \left( \sum_{y \in \Xspace} e^{-\eta_t L_{x(h),t,y}}   \right)
\end{align}

where the proportionality constant is set such that $\sum_{h'=0}^D q_t(h;\eta_t,\gprop) = 1$.
The quantity $q_t(h;\eta_t,\gprop)$ is exactly the \textit{posterior probability} that the algorithm \alg{ContextTreeAdaHedge$(D)$} selects a $h^{th}$-order model.
We will see that controlling the posterior on model order selection is crucial to bounding the variance in our desired manner.

First, we state a simple lemma that bounds Equation~\eqref{eq:qthactual} in terms of more intuitive quantities.

\begin{lemma}\label{lem:Qthbound}
We have
\begin{align}
\exp\{ - \eta_t \pihat_h(t) t + \ln \gprop(h) \} \leq Q_t(h;\eta_t,\gprop) \leq \exp\{ - \eta_t \pihat_h(t) t + 2^h \ln 2 + \ln \gprop(h) \} .
\end{align}
\end{lemma}

\begin{proof}
For the upper bound, we have
\begin{align*}
Q_t(h;\eta_t,\gprop) &:= \gprop(h) \prod_{x(h) \in \Xspace^h} \left( \sum_{y \in \Xspace} e^{-\eta_t L_{x(h),t,y}}\right) \\
&= \exp\left\{\sum_{x(h) \in \Xspace^h} \ln \left(  \sum_{y \in \Xspace} e^{-\eta_t L_{x(h),t,y}}\right) + \ln \gprop(h) \right\} \\
&\leq \exp\left\{\sum_{x(h) \in \Xspace^h} \ln \left(  2e^{-\eta_t  \min_{y \in \Xspace} \{L_{x(h),t,y}\}   }\right) + \ln \gprop(h) \right\} \\
&= \exp\left\{ - \sum_{x(h) \in \Xspace^h} \eta_t \min_{y \in \Xspace}\{L_{x(h),t,y}\} + 2^h \ln 2 + \ln \gprop(h)     \right\} \\
&= \exp\left\{ - \eta_t \pihat_h(t) t + 2^h \ln 2 + \ln \gprop(h)     \right\} 
\end{align*}

and for the lower bound, we have
\begin{align*}
Q_t(h;\eta_t,\gprop) &:= \exp\left\{\sum_{x(h) \in \Xspace^h} \ln \left(  \sum_{y \in \Xspace} e^{-\eta_t L_{x(h),t,y}} \right) + \ln \gprop(h) \right\} \\
&\geq \exp\left\{\sum_{x(h) \in \Xspace^h} \ln \left(e^{-\eta_t  \min_{y \in \Xspace} \{L_{x(h),t,y}\}   }\right) + \ln \gprop(h) \right\} \\
&= \exp\left\{ - \sum_{x(h) \in \Xspace^h} \eta_t \min_{y \in \Xspace}\{L_{x(h),t,y}\} + \ln \gprop(h)     \right\} \\
&= \exp\left\{ - \eta_t \pihat_h(t) t + \ln \gprop(h) \right\} 
\end{align*}
\end{proof}

Substituting $\ln \gprop(h) = -2^{h+1} \ln 2 = -2 \cdot 2^h \ln 2$, we get 
\begin{align}\label{eq:Qthbound}
\exp\{ - \eta_t \pihat_h(t) t - 2 \cdot 2^h \ln 2 \} \leq Q_t(h;\eta_t,\gprop) \leq \exp\{ - \eta_t \pihat_h(t) t - 2^h \ln 2 \} .
\end{align}

Equation~\eqref{eq:Qthbound} effectively makes the tradeoff between approximation error (reflected by the quantity $\pihat_h(t)$) and model complexity (reflected by the quantity $2^h \ln 2$ clear in the model-order selection problem.
We can think of the model orders as ``meta-experts" that are being randomized over.
Note that the learning rate that is being used to randomize their selection is still $\eta_t$!

\subsubsection{Analysis for a higher-than-needed model order}\label{sec:trueorder}

Here, we analyze the contribution of a specific selected model order to the variance, an important intermediate step.
Formally, we consider the algorithm \alg{ContextTreeAdaHedge$(h)$} equipped with the uniform prior function $\gunif(h') = \Ind[h'= h]$.
The regret guarantee is given by the following proposition.

\begin{proposition}\label{prop:uniform}
\begin{enumerate}
\item For any sequence $\{X_t,Y_t\}_{t =1}^T$ the algorithm \alg{ContextTreeAdaHedge$(h)$} with uniform prior gives us regret rate
\begin{align}
R_{T,d} = \Oh\left(\sqrt{T} 2^h \right) 
\end{align}

with respect to the best $d^{th}$-order tree expert in hindsight, and for every $d \leq h$.
\item \alg{ContextTreeAdaHedge$(h)$} with uniform prior gives
regret with probability greater than $(1 - \epsilon)$:
\begin{align*}
R_{T,d} &= \Oh\Big(\frac{2^{2h}}{(2\betastar - 1)^2} \left(h + \ln
\left(\frac{1}{\epsilon(2\betastar - 1)}\right)\right)\Big) .
\end{align*}
on a sequence $(X_t,Y_t)_{t \geq 1}$ that satisfies the $d^{th}$-order stochastic condition with parameter $\betastar$.
\end{enumerate}
\end{proposition}

Observe the suboptimal scaling in terms of $2^{2h}$ in the regret bound for the case where $d < h$.
We now proceed to prove Proposition~\ref{prop:uniform}.

Formally, the algorithm \alg{ContextTreeAdaHedge$(h)$} equipped with the uniform prior function $\gunif(h') = \Ind[h'= h]$ gives us $q_t(h';\eta_t,\gunif) = \Ind[h' = h]$, and we would get
\begin{align*}
\sum_{t=1}^T \sum_{h'=0}^D q_t(h';\eta_t,\gunif) w^{(h)}_{t,1 - Y_t^*}(\eta_t)  &= \sum_{t=1}^T w^{(h)}_{t,1 - Y_t^*} \\
&= \sum_{t=1}^T \frac{e^{-\eta_t D_t(h)}}{1 + e^{-\eta_t D_t(h) }} \\
&\leq \sum_{t = 1}^T \min\{e^{-\eta_t D_t(h)}, 1\} \\
&\leq \sum_{t=1}^T \min\{e^{-\eta_T D_t(h)}, 1\}
\end{align*}

where $D_t(h)$ is the gap between predictions as in Equation~\eqref{eq:dtcontexts}, and the last inequality is because $\eta_1^T$ is a decreasing sequence according to the update in Equation~\eqref{eq:etat}.

Therefore, we have
\begin{align}\label{eq:vtinequation}
V_T(\eta_1^T;\gunif) \leq \sum_{t = 1}^T \min\{e^{-\eta_t D_t(h)}, 1\} .
\end{align}

We observe that Equation~\eqref{eq:vtinequation} can be effectively unraveled to get a closed-form variance bound for particular evolutions of $\{D_t(h)\}_{t \geq 1}$.
Particularly, we care about $D_t(h)$ as a function of $N_t(X_t(h))$, the number of appearances so far of the current context.
We show this result in the following lemma.

\begin{lemma}\label{lem:unravelingvtinequation}
Let the following condition hold for some $t_0(h) > 0$ and $\alpha > 0$.
\begin{align}\label{eq:dtcondition}
D_t(h) &\geq \alpha N_t(X_t(h)) \text { for all $t$ such that } N_t(X_t(h)) \geq t_0(h)
\end{align}

for some $\alpha > 0$.

Then, we have
\begin{align}
\sum_{t=1}^{\infty} w^{(h)}_{t,1-Y_t^*}(\eta_t) &\leq 2^h \left(t_0(h) + \frac{1}{\eta_T \alpha} \right) .
\end{align}
\end{lemma}

\begin{proof}
We can directly use the condition in Equation~\eqref{eq:dtcondition}.
For values of $t$ such that $N_t(X_t(h)) < t_0(h)$, we apply $w^{(h)}_{t,1 - Y_t^*}(\eta_t) \leq 1$.
Otherwise, we use $w^{(h)}_{t,1 - Y_t^*}(\eta_t) \leq e^{-\eta_T \alpha N_t(X_t(h))}$.

Combining the two gives us 
\begin{align*}
\sum_{t=1}^{\infty} w^{(h)}_{t,1-Y_t^*}(\eta_t) &\leq \sum_{x(h) \in \Xspace^h} \left(t_0 + \sum_{s=t_0(h)}^{N_T(x(h))} e^{-\eta_T \alpha s}\right) \\
&\leq 2^h t_0(h) + \sum_{x(h) \in \Xspace^h} \sum_{s=t_0(h)}^{\infty} e^{-\eta_T \alpha s} \\
&\leq 2^h \left(t_0(h) + \sum_{s=t_0(h)}^{\infty} e^{-\eta_T \alpha s} \right) \\
&\leq 2^h \left(t_0(h) + \frac{e^{-\eta_T \alpha}}{1 - e^{-\eta_T \alpha}}\right ) .
\end{align*}

Now, we have
\begin{align*}
\frac{e^{-\eta_T \alpha}}{1 - e^{-\eta_T \alpha}} &= \frac{1}{e^{\eta_T \alpha} - 1} \\
&\leq \frac{1}{\eta_T \alpha}
\end{align*}

by the inequality $e^a \geq 1 + a$ for $a \geq 0$.
Substituting this above gives us our required result.
\end{proof}

It remains to show that the condition in Equation~\eqref{eq:dtcondition} is met with high probability for $(X_t,Y_t)_{t \geq 1}$ satisfying the $d^{th}$-order condition, for any $d \leq h$.
We use a standard Hoeffding-bounding technique to show this.

\begin{lemma}\label{lem:multiconcentration}
Let $\epsilon \in (0,1]$.
For a process $\{(X_t,Y_t)\}_{t \geq 1}$ satisfying the $d^{th}$-order stochastic condition with parameter $\betastar > 1/2$, the condition in Equation~\eqref{eq:dtcondition} holds for all $h \geq d$ for parameter values 
\begin{align}\label{eq:thigh}
\alpha &:= \frac{2\betastar - 1}{2} \\
t_0(h) = \thigh(h) &:= \frac{2}{\alpha^2} \ln \left(\frac{4(D-d)\cdot2^{h+1}}{\alpha^2 \epsilon}\right)
\end{align}

with probability greater than or equal to $(1 - \epsilon/2)$.
\end{lemma}

\begin{proof}
Essentially, we need to obtain to bound properties of the gap sequence $\{D_{t,(h)}\}_{t =1}^T$ so defined in Equation~\eqref{eq:dtcontexts} -- we use the Hoeffding bound for this.
This proof is a simple adaptation of the proof in the original AdaHedge paper~\cite{erven2011adaptive} to the case of contextual experts.

We denote the $p^{th}$ epoch of arrival of context $x(h) \in \Xspace^h$ by $T_p(x(h))$.
Showing that the condition in Equation~\eqref{eq:dtcondition} holds with probability greater than or equal to $(1 - \epsilon/2)$ is exactly equivalent to showing that the probability of the following bad event
\begin{align}\label{eq:badevent}
\Big\{\cup_{h=d}^D \cup_{\xvec_{(h)} \in \Xspace^h} \cup_{p = t_0(h)}^{N_T(x(h))} \left\{D_{T_p(x(h))}(h) < \alpha p\right\}\Big\}
\end{align}

is less than or equal to $\frac{\epsilon}{2}$.
We proceed by showing exactly this.


From the definition of a $d^{th}$-order stochastic process, we have $Y_t|X_t(d) \text{ i.i.d } \sim P^*(\cdot|X_t(d))$.
Therefore, we can write
\begin{align*}
D_{T_p(x(h))}(h) &= \sum_{s'= 1}^p 2 Z_{s'}
\end{align*}

where
\begin{align*}
\{Z_s'\}_{s' \geq 1}  \text{ i.i.d } \sim \begin{cases}
1 \text{ w. p. } \beta(x(d)) \\
-1 \text{ otherwise .} 
\end{cases}
\end{align*}

Denote $\alpha := \frac{2 \betastar - 1}{2}$.
We have $\EE[Z_s] = 2 \beta(x(d)) - 1 \geq 2\betastar - 1 = 2 \alpha$ and so we have $\EE[D_{T_p(x(h))}(h)] \geq 2 \alpha p$.
Noting that $Z_s \in \{-1,1\}$, we can directly use the Hoeffding bound to get
\begin{align*}
\Pr\left[D_{T_p(x(h))}(h) < \alpha p \right] &\leq \Pr\left[D_{T_p(x(h))}(h) < \left(\frac{2\beta(x(d)) - 1}{2}\right) p \right] \\
&\leq \exp\{-\frac{(2\beta(x(d)) - 1)^2 p}{8}\} \\
&\leq \exp\{-\frac{\alpha^2 p}{2}\} ,
\end{align*}

and so, for any $t_0(h) \geq 1$ and $x(h) \in \Xspace^h$, we can use the union bound to get
\begin{align*}
\Pr\left[ \cup_{p = t_0(h)}^{N_T(x(h))} \left\{D_{T_p(x(h))}(h) < \alpha p \right\} \right] &\leq \sum_{p = t_0(h)}^{N_T(x(h))} \exp\{-\frac{\alpha^2 p}{2}\}  \\
&\leq \sum_{p=t_0(h)}^{\infty} \exp\{-\frac{\alpha^2 p}{2}\} \\
&\leq \int_{u=t_0(h)}^{\infty} \exp\{-\frac{\alpha^2 u}{2}\} du \\
&= \frac{2e^{-\frac{\alpha^2 t_0(h)}{2}}}{\alpha^2} .
\end{align*}

We need to bound the probability that the above bad event happens \textit{ for any } context $x(h) \in \Xspace^h$ and model order $h \geq d$.
To do this, we apply the union bound twice more, to get
\begin{align*}
\Pr\left[ \cup_{h = d}^D \cup_{x(h) \in \Xspace^h} \cup_{p = t_0(h)}^{N_T(x(h))} \left\{D_{T_p(x(h))}(h) < \alpha p \right\} \right] &\leq \sum_{h=d}^D \sum_{x(h) \in \Xspace^h } \frac{2e^{-\frac{\alpha^2 t_0(h)}{2}}}{\alpha^2} \\
&= \left(\sum_{h=d}^D \frac{2\cdot 2^h \cdot e^{-\frac{\alpha^2 t_0(h)}{2}}}{\alpha^2}\right) \\
&\leq \epsilon/2
\end{align*}

if $t_0(h) \geq \thigh(h) = \frac{2}{(\alpha)^2}\ln \left(\frac{4(D-d)\cdot 2^h}{\epsilon (\alpha)^2}\right)$.


Setting $t_0(h) = \thigh(h)$ bounds the probability of the bad event as defined in Equation~\eqref{eq:badevent}, and completes our proof.

\end{proof}

\subsubsection{Completing proof of Proposition~\ref{prop:uniform}}

Now, the proof of Proposition~\ref{prop:uniform} follows directly from Lemmas~\ref{lem:secondorderbound} and~\ref{lem:unravelingvtinequation}.
We denote as shorthand the following:
\begin{align*}
\Delta_T^{(h)} &= \Delta_T(\eta_1^T;\gunif) \\
V_T^{(h)} &= V_T(\eta_1^T;\gunif)
\end{align*}
Substituting $g(\cdot) = \gunif(\cdot)$ into Lemma~\ref{lem:secondorderbound}, we have

\begin{align*}
R_{T,d} \leq R_{T,h} &\leq \left(\sqrt{V_T^{(h)} \ln 2} + \frac{2}{3} \ln 2 + 1\right)\left(1 + \frac{\ln \left(\frac{Z}{\gunif(h)}\right)}{\ln 2} \right) \\
&\leq \left(\sqrt{V_T^{(h)} \ln 2} + \frac{2}{3} \ln 2 + 1\right)\left(1 + 2^h \right)
\end{align*}

Thus, it remains to bound the variance term $V_T^{(h)}$.
We denote the final learning rate as
\begin{align*}
\eta^{(h)}_T = \frac{\ln 2}{\Delta^{(h)}_{T-1}} \geq \frac{\ln 2}{\Delta^{(h)}_T}
\end{align*}
and from~\cite{de2014follow} that
\begin{align*}
\Delta^{(h)}_T &\leq \sqrt{V^{(h)}_T \ln 2} + \frac{2}{3} \ln 2 + 1 \\
&\leq \sqrt{V^{(h)}_T} \Big(\sqrt{\ln 2} + \frac{4}{3}\ln 2 + 2\Big) \Big{(}\text { as } \sqrt{V^{(h)}_T} \geq \sqrt{v^{(h)}_1} = \frac{1}{2}\Big{)}\\
&\leq 6 \sqrt{V^{(h)}_T} \ln 2 .
\end{align*}

Together, these give us 
\begin{align*}
\eta^{(h)}_T \geq \frac{1}{6\sqrt{V^{(h)}_T}}
\end{align*}

and therefore, we have with probability greater than or equal to $(1 - \epsilon)$,
\begin{align*}
V_T^{(h)} &\leq \sum_{t=1}^T w^{(h)}_{t,1- X^*_t} \\
&\leq 2^h\left(\thigh(h) + \frac{1}{\eta^{(h)}_T (2\betastar - 1)} \right)\\
&\leq 2^h\left(\thigh(h) + \frac{6\sqrt{V^{(h)}_T}}{(2\betastar - 1)}\right) \\
&\leq 2^h\left(\frac{8}{(2\betastar - 1)^2}\ln\left(\frac{8 \cdot 2^h}{\epsilon(2\betastar - 1)^2}\right) + \frac{6\sqrt{V^{(h)}_T}}{(2\betastar - 1)}\right) 
\end{align*}

Therefore, we have
\begin{align*}
\sqrt{V^{(h)}_T} &\leq \frac{8 \cdot 2^h}{(2\betastar - 1)^2}\ln\left(\frac{8 \cdot 2^h}{\epsilon(2\betastar - 1)^2}\right) + \frac{6\cdot 2^h}{(2\betastar - 1)} \\
&\leq \frac{14 \cdot 2^h}{(2\betastar - 1)^2}  \ln\left(\frac{8 \cdot 2^h}{\epsilon(2\betastar - 1)^2}\right)
\end{align*}

This gives us 
\begin{align*}
R_{T,d} &= \Oh\Big(\frac{2^{2h}}{(2\betastar - 1)^2} \left(h + \ln
\left(\frac{1}{\epsilon(2\betastar - 1)}\right)\right)\Big) .
\end{align*}

with probability greater than or equal to $(1- \epsilon)$.
This completes the proof.

\subsubsection{Ruling out higher-order models}\label{sec:higherorder}

We can make two clear inferences from Lemma~\ref{lem:unravelingvtinequation}:
\begin{enumerate}
\item \alg{ContextTreeAdaHedge$(d)$} gives us the true regret scaling in terms of $\Oh(2^{2d} \left(d + \ln\left(\frac{1}{\epsilon}\right) \right) )$.
\item For $h > d$, \alg{ContextTreeAdaHedge$(h)$} gives us suboptimal scaling $\\Oh(2^{2h} \left(h + \ln\left(\frac{1}{\epsilon}\right) \right) )$.
The reason for suboptimality is because of sample splitting: for every true context $x(d) \in \Xspace^d$, we are unnecessarily splitting the data into $2^{d-h}$ extra contexts and treating the best predictors for these contexts as independent.
\end{enumerate}

It is clear, particularly from the second inference, that we would like to control the posterior probability with which we select overly complex models.
This quantity is expressed as $q_t(h;\eta_t,\gprop)$ for all $h > d$.
Now, we consider an explicit upper bound on $q_t(h;\eta_t,\gprop)$ and show how it decreases with $t$.

Using Equation~\eqref{eq:Qthbound}, it is convenient to consider the following upper bound on the quantity $q_t(h;\eta_t,\gprop)$ for $h > d$:
\begin{align*}
q_t(h;\eta_t,\gprop) &= \frac{Q_t(h;\eta_t,\gprop)}{\sum_{h'=0}^D Q_t(h';\eta_t,\gprop)} \\
&\leq \frac{Q_t(h;\eta_t,\gprop)}{Q_t(d;\eta_t,\gprop)} \\
&\leq \exp\{ \eta_t (\pihat_d(t) - \pihat_h(t)) t - 2^h \ln 2 + 2 \cdot 2^d \ln 2 \}
\end{align*}

We should expect that as $t$ becomes large the difference in estimated approximation errors is negligible, i.e. we will observe that $\pihat_h(t) = \pihat_d(t)$ with high probability.
We would then get a scaling of $q_t(h;\eta_t,\gprop) \leq \exp\{- 2^h \ln 2\}$.
However, we can say $\pihat_h(t) = \pihat_d(t)$ with high probability only after $\Oh(2^h)$ rounds.
Before this, and particularly for times between $\Oh(2^d)$ and $\Oh(2^h)$, we have to worry about the difference in approximation errors, $\eta_t (\pihat_h(t) - \pihat_d(t)) t$.
This is the \textit{overfitting regime} in which the $h$th order model may look deceptively better.
Luckily, we can cap this quantity as well owing to already established statistical guarantees on the sequence $\{X_t\}_{t \geq 1}$.
The following lemma expresses this.

\begin{lemma}\label{lem:overfitting}
The process $\{(X_t,Y_t)\}_{t \geq 1}$ satfisfying Equation~\eqref{eq:dtcondition} for all $h \geq d$ and for 
\begin{align*}
t_0(h) = \thigh(h)
\end{align*}

directly implies
\begin{align}\label{eq:overfittingcap}
(\pihat_d(t) - \pihat_h(t))t \leq 2^{h-1} \thigh(h) .
\end{align}

\end{lemma}

\begin{proof}

Recall the notation we defined for the best $d^{th}$-order tree expert at time $t$, $\widehat{F}_d(t)$, as well as the number of appearances of context $x(h)$ at time $t$, denoted by $N_t(x(h))$.

From Definition~\ref{def:pihat}, we have 
\begin{align*}
&(\pihat_d(t) - \pihat_h(t)) t \\
&= \sum_{x(d) \in \Xspace^d} N_t(x(d)) \left(1 - \max_{y \in \Xspace} \{\Phat_t(y|x(d)) \}\right)- \sum_{x(h) \in \Xspace^h} N_t(x(h)) \left(1 - \max_{y \in \Xspace} \{\Phat_t(y|x(h)) \}\right) \\
&= \sum_{x(d) \in \Xspace^d} \underbrace{\left(\sum_{x(h):x(d) \subset x(h)} N_t(x(h)) \left(\max_{y \in \Xspace} \{\Phat_t(y|x(h)) \} -\Phat_t(\widehat{F}_d(t)(x(d))|x(d)) \right)\right)}_{T_1}
\end{align*}

Let $T_1$ be the quantity under the brace (for shorthand).
We also define the number of \textit{super-contexts} of length $h$ that contain $x(d)$,
\begin{align*}
S_{t,h-d}(x(d)) := \sum_{x(h):x(d) \subset x(h)} \Ind[N_t(x(h)) > 0] .
\end{align*}

Now, we have one of two cases:
\begin{enumerate}
\item We have $N_t(x(d)) \leq t_{\mathsf{high}}$.
In this case, we have $T_1 \leq \frac{\thigh}{2}$.
\item $N_t(x(d)) > \thigh$.
In this case, we have $\widehat{F}_d(t)(x(d)) = f^*(x(d))$ from Equation~\eqref{eq:dtcondition}, and we directly get
\begin{align*}
T_1 &= \sum_{x(h): x(d) \subset x(h) \text{ and } {\arg \max} \{\Phat_t(y|x(h))\} \neq f^*(x(d))} N_t(x(h)) \left(\max_{y \in \Xspace} \{\Phat_t(y|x(h)) \} -\Phat_t(f^*_d(x(d))|x(d))\right)
\end{align*}

Clearly, the overfitting effect is created \textit{only} by the set of contexts $x(h)$ for which the best predictor does not match $f^*(x(d))$.
From Lemma~\ref{lem:multiconcentration}, Equation~\eqref{eq:dtcondition} is satisfied for all $h \geq d$ and for $N_t(x(h)) \geq \thigh(h)$.
It is easy to see that Equation~\eqref{eq:dtcondition} implies a non-negative separation between the truly correct predictor $f^*(x(d))$ and its alternative, and so we have
\begin{align*}
{\arg \max}_{y \in \Xspace} \{\Phat_t(y|x(h))\} = f^*(x(d)) \text{ if } N_t(x(h)) \geq \thigh(h) .
\end{align*}

Substituting this directly, and noting that 
\begin{align*}
\max_{y \in \Xspace} \{\Phat_t(y|x(h)) \} -\Phat_t(f^*_d(x(d))|x(d)) \leq 1/2
\end{align*} 

gives us
\begin{align*}
T_1 &\leq \sum_{x(h): x(d) \subset x(h)  \text{ and }  N_t(x(h)) \leq \thigh(h)} \frac{\min\{N_t(x(h),\thigh(h)\}}{2}\} \\
&\leq \sum_{x(h): x(d) \subset x(h)  \text{ and }  N_t(x(h)) \leq \thigh(h)} \frac{\thigh(h)}{2} \\
&\leq S_{t,h-d}(x(d)) \frac{\thigh(h)}{2} .
\end{align*}
\end{enumerate}

Noting that $1 \leq 2^{h-d}$ and $S_{t,h-d}(x(d)) \leq 2^{h-d}$ gives us
\begin{align*}
T_1 \leq 2^{h-d} \frac{\thigh(h)}{2} , 
\end{align*}

and substituting back this expression yields



\begin{align*}
(\pihat_d(t) - \pihat_h(t))t &\leq \sum_{x(d) \in \Xspace^d} T_1 \\
&\leq 2^{h-1} \thigh(h) .
\end{align*}

This completes our proof.

\end{proof}

Recall that for all $t > T_0(h)$ where $T_0(h)$ is as defined in Lemma~\ref{lem:thresholding} with respect to $t_0(h) = \thigh(h)$, we have $\eta_t < \frac{\ln 2}{t_0}$.
Under this condition, the explicit cap on the overfitting effect as defined in Lemma~\ref{lem:overfitting}, together with the adaptive regularization of \alg{AdaHedge}, ensures that we can sufficiently restrict the contribution of higher-order models. 

We use Equation~\eqref{eq:overfittingcap} to get
\begin{align*}
q_t(h;\eta_t,\gprop) &\leq \exp\{ \eta_t (\pihat_d(t) - \pihat_h(t)) t -  2^h \ln 2 + 2 \cdot 2^d \ln 2 \} \\
&\leq \exp\{\frac{2^{h-1}\thigh(h) \ln 2}{\thigh(h)} - 2^h \ln 2 + 2^{d+1} \ln 2\} \\
&\leq \exp\{- 2^{h-1} \ln 2 + 2^{d+1} \ln 2\} \\
&= 2^{-2^{h-1} + 2^{d+1}} .
\end{align*}

Therefore, we can apply Lemma~\ref{lem:unravelingvtinequation} to get

\begin{align*}
\sum_{t=T_0}^{T} q_t(h;\eta_t,\gprop) w^{(h)}_{t,1-Y_t^*}(\eta_t) &\leq 2^{-2^{h-1} + 2^{d+1}} \sum_{t=T_0}^T w^{(h)}_{t,1-Y_t^*} \\
&\leq  2^{h - 2^{h-1} + 2^{d+1}} \left(\thigh(h) + \frac{1}{\eta_T \alpha} \right) .
\end{align*}

It is now easy to check that 
\begin{align*}
2h &\leq 2^{h-1} - 2^{d+1} \text { for all } h \geq d + 4  \text{ and } d \geq 0 \\
\implies h - 2^{h-1} + 2^{d+1} &\leq -h \\
\implies 2^{h - 2^{h-1} + 2^{d+1}} &\leq 2^{-h} .
\end{align*}

Therefore, for $h \geq d + 4$, we get
\begin{align*}
\sum_{t=T_0}^{T} q_t(h;\eta_t,\gprop) w^{(h)}_{t,1-Y_t^*}(\eta_t) &\leq 2^{-h} \left(\thigh(h) + \frac{1}{\eta_T \alpha}  \right) .
\end{align*}

For $h < d + 4$, we do not try to non-trivially bound $q_t(h;\eta_t,\gprop)$.
We directly use Lemma~\ref{lem:unravelingvtinequation} to get 
\begin{align*}
\sum_{t=T_0}^{T} q_t(h) w^{(h)}_{t,1-Y_t^*}(\eta_t) &\leq 2^h \left(\thigh(h) + \frac{1}{\eta_T \alpha}  \right) .
\end{align*}

We have thus guaranteed that the contribution from the higher-order models (particularly for $h \geq d + 4$) not only has no exponential dependence on $h$, but is in fact exponentially decaying in $h$!
Ultimately, we will see that we get a very weak linear dependence on $D$, the maximum model order, in our regret bound.

\subsubsection{Ruling out \textit{bad} lower-order models}\label{sec:lowerorder}

Using Equation~\eqref{eq:Qthbound}, it is convenient to consider the following upper bound on the quantity $q_t(h)$ for $h < d$:
\begin{subequations}
\begin{align}
q_t(h;\eta_t,\gprop) &\leq \frac{Q_t(h;\eta_t,\gprop)}{Q_t(d;\eta_t,\gprop)}  \\
&\leq \exp\{ -\eta_t (\pihat_h(t) - \pihat_d(t)) t  + 2 \cdot 2^d \ln 2 - 2^h \ln 2\} \label{eq:qthlower}
\end{align}
\end{subequations}

Ruling out lower-order models actually stems from the fact that we can make concrete statements about the sequence's unpredictability (poor approximability) under these models.

The kind of concrete statement that we would like is detailed in the lemma below.

\begin{lemma}\label{lem:lowerorderpihatcondition}
Let $h < d$. 
Consider a sequence $\{x_t\}_{t \geq 1}$ such that we have 
\begin{align}\label{eq:lowerorderpihatcondition}
(\pihat_h(t) - \pihat_d(t))t \geq \alpha_{h,d} t \text { for all } t \geq t_0(h) > 0 
\end{align}

for some $\alpha_{h,d} > 0$.

Then, we have
\begin{align}\label{eq:lowerordercontribution}
\sum_{t=1}^T q_t(h;\eta_t,\gprop) w^{(h)}_{t,1-Y_t^*}(\eta_t) &\leq \tlow'(h) + \frac{1}{\eta_T \alpha_{h,d} }
\end{align}

where
\begin{align}\label{eq:t0dash}
\tlow'(h) = \max\{t_0(h), \frac{2 \cdot 2^d \ln 2}{\eta_T \alpha_{h,d}}\} .
\end{align}
\end{lemma}

\begin{proof}
The condition in Equation~\eqref{eq:lowerorderpihatcondition} is essentially the same as the condition on gaps between losses in the original AdaHedge paper~\cite{erven2011adaptive} used to prove constant regret bounds.
We use a similar argument here.

First, we subsitute the condition in Equation~\eqref{eq:lowerorderpihatcondition} into Equation~\eqref{eq:qthlower} to get the upper bound 
\begin{align*}
q_t(h;\eta_t,\gprop) &\leq \exp\{-\eta_t \alpha_{h,d} t + 2\cdot 2^d \ln 2 - 2^h \ln 2\} \\
&\leq \exp\{-\eta_t \alpha_{h,d} t + 2 \cdot 2^d\} \\
&= \exp\{2 \cdot 2^d \ln 2 - \eta_t \alpha_{h,d} t\} \\
&\leq \exp\{2 \cdot 2^d \ln 2 - \eta_T \alpha_{h,d} t\} .
\end{align*}

where the last inequality applies because $\eta_1^T$ is a decreasing sequence.
Putting this together with the trivial bound $q_t(h;\eta_t,\gprop) \leq 1$ gives us 
\begin{align*}
q_t(h;\eta_t,\gprop) &\leq \begin{cases}
1 \text{ for } t \leq \tlow'(h) \\
\exp\{2 \cdot 2^d \ln 2 - \eta_T \alpha_{h,d} t\} \text { for } t > \tlow'(h) .
\end{cases}
\end{align*}

where we have
\begin{align*}
\tlow' = \max\{t_0(h), \frac{2\cdot 2^d \ln 2 }{\eta_T \alpha_{h,d}}\} .
\end{align*}

From this, using the trivial bound $w_{t,1 - Y_t^*}(\eta_t) \leq 1$ we get
\begin{align*}
\sum_{t=1}^T q_t(h;\eta_t,\gprop) w_{t,1 - Y_t^*}(\eta_t) &\leq \tlow'(h) + \sum_{t = \tlow'(h) + 1}^{\infty} \exp\{2 \cdot 2^d \ln 2 - \eta_T \alpha_{h,d} t\} \\
&\leq  \tlow'(h) + \exp\{2 \cdot 2^d \ln 2 - \eta_T \alpha_{h,d}\tlow'(h)\}\left(\sum_{t=1}^{\infty} e^{-\eta_T  \alpha_{h,d} t}\right) \\
&= \tlow'(h) + \sum_{t=1}^{\infty} e^{- \eta_T \alpha_{h,d} t} \\
&\leq \tlow'(h) + \int_{u = 0}^{\infty} e^{- \eta_T \alpha_{h,d} u} du \\
&= \tlow'(h) + \frac{1}{\eta_T \alpha_{h,d}} \int_{v=0}^{\infty} e^{-v} dv \\
&= \tlow'(h) + \frac{1}{\eta_T \alpha_{h,d}} ,
\end{align*}

This completes the proof.
\end{proof}

From Lemma~\ref{lem:lowerorderpihatcondition}, we can clearly bound the contribution of lower-order models to cumulative variance by a constant term.
This is because the difference in estimated unpredictability between the right model and the bad lower-order model remains as the number of rounds increase -- leading to an exponentially decaying likelihood of selecting the lower-order model.
(We do not even need to use any information about whether the online learning algorithm would ensure low regret when selecting a lower-order model, although this is sometimes the case in practice\footnote{In fact, models that are close in approximability to the true model will suffer less regret. Ideally, our analysis should consider this nuance, but doing so is likely to be technically challenging because of the data-dependent learning rate.}.)

It is therefore of interest to understand when the condition in Equation~\eqref{eq:lowerorderpihatcondition} holds, and in particular, characterize $\tlow'(h)$.
Recall the definition of asymptotic unpredictability 
\begin{align}
\pi^*_h := \sum_{x(h) \in \Xspace^h } Q^*(x(h)) \left[1 - \max_{y \in \Xspace}\{P^*(y|x(h))\}\right]
\end{align}

Also recall that for $h > d$, we have $\pi^*_h = \pi^*_d$; and for $h < d$, we have $\pi^*_h > \pi^*_d$.
It is also well-known~\cite{feder1992universal} that  
\begin{align*}
\pihat_h(t) \xrightarrow{\mathsf{prob.}} \pistar_h \text { for all } h \in \{0,1,\ldots,D\} .
\end{align*}

So the intuition is that for a large enough value of $t$, we should also start to see a \textit{strict} decaying in the estimated unpredictability as $h$ increases to $d$ -- and we should be able to rule out the poorly performing $h$th order models when $h < d$.
That is,
\begin{align*}
\pihat_h(t) > \pihat_d(t) \text{ for all } h < d .
\end{align*}

We formalize this intuition in the lemma below.

\begin{lemma}\label{lem:convergencepihat}
Let $\{(X_t,Y_t)\}_{t \geq 1}$ satisfy the $d^{th}$-order stochastic condition.
Then, Equation~\eqref{eq:lowerorderpihatcondition} holds for all $d < h$ with probability greater than equal to $(1 - \epsilon/2)$ and with parameters
\begin{subequations}\label{eq:alphaandtlow}
\begin{align}
\alpha_{h,d} &= \frac{\pistar_h - \pistar_d}{2} \label{eq:alpha}\\
t_0(h) = \tlow(h) &:= \frac{32d}{\alpha_{h,d}^2} \left(d \cdot 2^h \ln 2 + \ln \left(\frac{64d}{\epsilon \alpha_{h,d}^2}\right)\right) . \label{eq:tlow}
\end{align}
\end{subequations}
\end{lemma}


\begin{proof}

Recall our notation for the class of Boolean functions from $\Xspace^h$ to $\Xspace$, denoted by $\mathcal{F}_h$.
We can express each of the unpredictability estimates $\pihat_h(t)$ as a minimum of $|\mathcal{F}_h|$ Lipschitz functions, as follows.

\begin{align*}
t \pihat_h(t) &= \min_{\boldf \in \mathcal{F}_h} \Big{\{}f_{(h)}(\{(X_s,Y_s)\}_{s = 1}^t; \boldf) \Big{\}} \text{ where } \\
f_{(h)}(\{(X_s,Y_s)\}_{s = 1}^t;\boldf) &:= \sum_{s = 1}^t \Ind[Y_s \neq \boldf(X_s(h))] \\
&= \sum_{s=1}^t Z_s 
\end{align*}

where $Z_s = \Ind[Y_s \neq \boldf(X_s(h))]$.
Note that $\{Z_s\}_{s=1}^t$ are independent variables taking values in $\{0,1\}$.
Therefore, the standard Hoeffding bound gives us
\begin{align}\label{eq:tailboundgh}
\Pr\left[|f_{(h)}(\{(X_s,Y_s)\}_{s = 1}^t;\boldf) - \EE[f_{(h)}(\{(X_s,Y_s)\}_{s = 1}^t;\boldf)]| > \delta t  \right] &\leq 2\exp\{-\frac{\delta^2 t}{2}\} .
\end{align}

Observe that $\pihat_h(t)$ itself is not an unbiased estimate of $\pistar_h$.
But we know that 
\begin{align*}
\EE[t \pihat_h(t)] = \EE\left[\min_{\boldf \in \mathcal{F}_h} f_{(h)}(\{(X_s,Y_s)\}_{s = 1}^t;\boldf) \right] \leq \EE[f_{(h)}(\{(X_s,Y_s)\}_{s = 1}^t;\boldf^*_h)] = t\pistar_h
\end{align*} 

for all $\boldf \in \mathcal{F}_h$.
The upper tail bound therefore follows easily -- from Equation~\eqref{eq:tailboundgh}, we have
\begin{align*}
\Pr\left[ t\pihat_h(t) - t\pistar_h > \delta t \right] &\leq \Pr\left[ f_{(h)}(X^t;\boldf^*_{h}) - \EE[f_{(h)}(X^t;\boldf^*_{h})]\right] \\
&\leq \exp\{-\frac{\delta^2 (1 - \gamma)^2 t}{2d}\} .
\end{align*}

To get the lower tail bound, we need to use the union bound.
\begin{align*}
\Pr\left[ t\pistar_h - t\pihat_h(t) > \delta t \right] &= \Pr\left[ t \pihat_h(t) < t \pistar_h - \delta t \right] \\
&\leq \sum_{\boldf \in \mathcal{F}^h} \Pr\left[ f_{(h)}(X^t;\boldf) < t \pistar_h - \delta t \right] \\
&= \sum_{\boldf \in \mathcal{F}^h} \Pr\Big{[} f_{(h)}(\{(X_s,Y_s)\}_{s = 1}^t;\boldf) - \EE[f_{(h)}(\{(X_s,Y_s)\}_{s = 1}^t;\boldf)] <\\
t \pistar_h - &\EE[f_{(h)}(\{(X_s,Y_s)\}_{s = 1}^t;\boldf)] - \delta t \Big{]} \\
&\leq \sum_{\boldf \in \mathcal{F}^h} \Pr\left[ f_{(h)}(\{(X_s,Y_s)\}_{s = 1}^t;\boldf) - \EE[f_{(h)}(\{(X_s,Y_s)\}_{s = 1}^t;\boldf)] < - \delta t \right] \\
&\leq 2^{2^h} \exp\{-\frac{\delta^2 t }{2}\} .
\end{align*}

Next, we plug in $\delta = \alpha_{h,d} = \frac{\pistar_h - \pistar_d}{4}$ and re-apply the union bound to get
\begin{align*}
&\Pr\left[ \cup_{h = 0}^{d-1} \{(\pihat_h(t) - \pihat_d(t)) \leq \alpha_{h,d} \text{ for some } t \geq t_0(h)\}\right] \\
&\leq \Pr\left[\cup_{h=0}^{d-1} \{\pistar_h - \pihat_h(t) \leq \frac{\alpha_{h,d}}{2} \text{ for some } t \geq t_0(h) \} \cup \{\pihat_d(t) - \pistar_d \leq \frac{\alpha_{h,d}}{2} \text{ for some } t \geq t_0(h)\}\right] \\
&\leq \sum_{h=0}^{d-1} \Pr\left[\pistar_h - \pihat_h(t) \leq \frac{\alpha_{h,d}}{2} \text{ for some } t \geq t_0(h)\right] + \Pr\left[\pihat_d(t) - \pistar_d \leq \frac{\alpha_{h,d}}{2} \text{ for some } t \geq t_0(h)\right] \\
&\leq \sum_{h=0}^{d-1} \frac{32 \cdot 2^{2^h}}{ \alpha_{h,d}^2}e^{-\frac{\alpha_{h,d}^2  t_0(h) }{32}} + \frac{32}{\alpha_{h,d}^2}e^{-\frac{\alpha_{h,d}^2 t_0(h) }{32}} \\
&\leq \epsilon/2 \text{ when } \\
t_0(h) &\geq \tlow(h) := \frac{32}{\alpha_{h,d}^2} \left(d \cdot 2^h \ln 2 + \ln \left(\frac{64d}{\epsilon \alpha_{h,d}^2}\right)\right) .
\end{align*}

This completes our proof.

\end{proof}

\subsubsection{Putting the pieces together: Proof of Theorem~\ref{thm:contexttreeadahedge}}

In Section~\ref{sec:trueorder}, we determined the overall contribution to the cumulative variance coming from the vicinity of the true model orders, $h \in \{d,d+1,d+2,d+3\}$.
Then, in Section~\ref{sec:higherorder} +~\ref{sec:lowerorder}, we appropriately limited the contribution of lower-order and higher-order models to the cumulative variance.
Now, we put together the pieces and characterize cumulative regret to complete the proof of Theorem~\ref{thm:contexttreeadahedge}.

First, we apply Lemma~\ref{lem:thresholding} setting $t_0 = \thigh(D)$.
Recall that $\thigh(D)$ represents the number of appearances of a full context before which we cannot necessarily make statistical guarantees about the predictor.
This gives us\footnote{Equation~\eqref{eq:deltatthreshold} exposes new conceptual beauty in the umbrella of approaches to varying the learning rate inversely proportional to accumulated regret so far.
The only reason a high learning rate does not affect us is because it means that very little regret has been accumulated up to that point.
Effectively, $t_0 = \thigh(D)$ represents the extent of cumulative mixability the algorithm is willing to tolerate in this regime before carrying out probabilistic stochastic model selection, and is the natural statistical quantity to reflect this.
}
\begin{align}\label{eq:deltatthreshold}
\Delta_T \leq \thigh(D) + \sqrt{V_{T_0(D)}^T \ln 2} + \frac{2}{3} \ln 2 + 2 .
\end{align}

We now proceed to bound the quantity $V_{T_0(D)}^T$.
Recall that
\begin{align*}
V_{T_0(D)}^T &\leq \sum_{h=0}^D q_t(h) \sum_{t=T_0(D)}^T w^{(h)}_{t,1-X^*_t} \\
&\leq \underbrace{\sum_{h=0}^{d-1} q_t(h) \sum_{t=T_0(D)}^T w^{(h)}_{t,1-X^*_t}}_{T_1} + \underbrace{\sum_{h=d}^{d+3} \sum_{t=T_0(D)}^T w^{(d)}_{t,1-X^*_t}}_{T_2} + \underbrace{\sum_{h=d+4}^D q_t(h) \sum_{t=T_0(D)}^T w^{(h)}_{t,1-X^*_t}}_{T_3}\\
\end{align*}

We start with summarizing the lower-order model contribution $T_1$.
From Lemmas~\ref{lem:lowerorderpihatcondition} and~\ref{lem:convergencepihat}, we have
\begin{align*}
T_1 &\leq \sum_{h=0}^{d-1} \tlow'(h) + \frac{1}{\eta_T} \left(\sum_{h=0}^{d-1} \frac{1}{\alpha_{h,d}}\right) \\
&\leq d \tlow'(d-1) + \frac{1}{\eta_T} \left(\sum_{h=0}^{d-1} \frac{1}{\alpha_{h,d}}\right) .
\end{align*}

\textit{Notice that $T_1$ is a constant independent of the horizon $T$ as long as $\eta_T$ does not decay with $T$.}

Next, we move on to the vicinity of the true model order contribution, represented by model orders $\{d,d+1,d+2,d+3\}$.
From Lemmas~\ref{lem:unravelingvtinequation} and~\ref{lem:multiconcentration}, we get
\begin{align*}
T_2 &\leq \sum_{h=d}^{d+3} 2^h \left(\thigh(h) + \frac{1}{\eta_T (2\betastar - 1) }\right) \\
&\leq 15 \cdot 2^d \left( \thigh(d + 3) + \frac{1}{\eta_T (2\betastar - 1)} \right) .
\end{align*}

\textit{Notice that $T_2$ is roughly what we should expect (upto constant factors) if we knew the model order exactly.}

Finally, we summarize the higher-order-model contribution $T_3$.
From Lemma~\ref{lem:overfitting} and the analysis in Section~\ref{sec:higherorder}, we have
\begin{align*}
T_3 &\leq \sum_{h=d+4}^D 2^{-h} \left(\thigh(h) + \frac{1}{\eta_T (2\betastar - 1) }\right) \\
&= \sum_{h=d+4}^D 2^{-h} \thigh(h) + \frac{2}{\eta_T(2\betastar - 1)} .
\end{align*}

Recall from Equation~\eqref{eq:thigh} that 
\begin{align*}
\thigh(h) &= \frac{2}{(2\betastar - 1)^2} \ln \left(\frac{(D-d) \cdot 2^h}{(2\betastar - 1)^2 \epsilon}\right) \\
&= \frac{2h}{(2\betastar - 1)^2} \ln 2 + \frac{2}{(2\betastar - 1)^2} \ln \left(\frac{(D-d)}{(2\betastar - 1)^2 \epsilon}\right)
\end{align*}

and since $\sum_{h=0}^{\infty} 2^{-h} \leq \sum_{h=0}^{\infty} h \cdot 2^{-h} = 4$, we get
\begin{align*}
T_3 &\leq \frac{8}{(2\betastar - 1)^2} \ln 2 + \frac{8}{(2\betastar - 1)^2} \ln \left(\frac{(D-d)}{(2\betastar - 1)^2 \epsilon}\right) + \frac{2}{\eta_T(2\betastar - 1)} = 8\thigh(1) + \frac{2}{\eta_T(2\betastar - 1)} .
\end{align*}

\textit{Notice that $T_3$ is a constant that scales only logarithmically in the maximum model order $D$!}

Now combining the three equations for $T_1$,$T_2$ and $T_3$, we get
\begin{align*}
V_{T_0(D)}^T &\leq d \tlow'(d-1) + 15 \cdot 2^d \thigh(d + 3) + 8\thigh(1) + \frac{(d + 1) \cdot 2^d}{\eta_T \overline{\gamma}} , 
\end{align*}

where
\begin{align*}
\frac{1}{\overline{\gamma}} := \frac{1}{d + 1} \Big(\sum_{h=0}^{d-1} \frac{1}{\alpha_{h,d}} + \frac{15}{(2\betastar - 1)})\Big .
\end{align*}

Next, recall from Equation~\eqref{eq:t0dash} that 
\begin{align*}
\tlow'(d-1) = \max\{\tlow(d-1), \frac{2 \cdot 2^d}{\eta_T \alpha_{d-1,d} }\} 
&\leq \tlow(d-1) + \frac{2\cdot 2^d}{\eta_T \alpha_{d-1,d}}
\end{align*}

using Fact~\ref{f:maxlessthansum}.
Substituting this expression gives us
\begin{align*}
V_{T_0(D)}^T &\leq d \cdot \tlow(d-1) + 15 \cdot 2^d \cdot \thigh(d + 3) + 8\thigh(1) + \frac{(d + 2) \cdot 2^d}{\eta_T \overline{\gamma}} .
\end{align*}

Next, we use the connection between learning rate and mixability gap from Equation~\eqref{eq:etat} to get 
\begin{align*}
\eta_T &= \frac{\ln 2}{\Delta_{T-1}} \geq \frac{\ln 2}{\Delta_T} \\
\implies \frac{1}{\eta_T} &\leq \frac{\Delta_T}{\ln 2} \\
&\leq \frac{\thigh(D)}{\ln 2} + \frac{1}{\ln 2} \left(\sqrt{V_{T_0(D)}^T \ln 2} + \frac{2}{3} \ln 2 + 1 \right) \\
\end{align*}

where in the last step we applied Equation~\eqref{eq:deltatthreshold}.

Ultimately, we get the following inequality for $V_{T_0(D)}^T$:
\begin{align*}
V_{T_0(D)}^T &\leq d \cdot \tlow(d-1 ) + 15 \cdot 2^d \cdot \thigh(d + 3) + 8 \thigh(1) + \frac{(d + 2) \cdot 2^d}{ \overline{\gamma}} \left( \frac{\thigh(D)}{\ln 2} + \frac{1}{\ln 2} \left(\sqrt{V_{T_0(D)}^T \ln 2} + \frac{2}{3} \ln 2 + 1 \right) \right) .
\end{align*}

Now, we have two cases:
\begin{enumerate}
\item $V_{T_0(D)}^T < \frac{1}{4}$.
\item $V_{T_0(D)}^T \geq \frac{1}{4}$, in which case, we get
\begin{align*}
V_{T_0(D)}^T &\leq \sqrt{V_{T_0(D)}^T} \Big{(} 2d \cdot \tlow(d-1) + 30 \cdot 2^d \cdot \thigh(d+ 3) + 16 \cdot \thigh(1) \\
&+ \frac{2 \cdot (d + 2) \cdot 2^d \cdot \thigh(D)}{\overline{\gamma} \ln 2} + \frac{1}{\sqrt{\ln 2}} + \frac{2}{3} + \frac{1}{\ln 2}\Big{)} \\
\implies \sqrt{V_{T_0(D)}^T} &\leq 2d \cdot \tlow(d-1) + 30 \cdot 2^d \cdot \thigh(d+ 3) + 16 \cdot \thigh(1) + \frac{2 \cdot (d + 2) \cdot 2^d \cdot \thigh(D)}{\overline{\gamma} \ln 2} \\
&+ \frac{1}{\sqrt{\ln 2}} + \frac{2}{3} + \frac{1}{\ln 2} .
\end{align*}

\end{enumerate}

So, we have bounded the cumulative variance term $V_{T_0(D)}^T$.
We now substitute back into Equation~\eqref{eq:deltatthreshold} to get
\begin{align*}
\Delta_T &\leq \thigh(D) + \Big(  2d \cdot \tlow(d-1) + 30 \cdot 2^d \cdot \thigh(d+ 3) + 16 \cdot \thigh(1) + \frac{2 \cdot (d + 2) \cdot 2^d \cdot \thigh(D)}{\overline{\gamma} \ln 2} \\
&+ \frac{1}{\sqrt{\ln 2}} + \frac{2}{3} + \frac{1}{\ln 2}\Big)\sqrt{\ln 2} + \frac{2}{3} \ln 2 + 2 .
\end{align*}

Observe, from this inequality, that the cumulative mixability gap $\Delta_T$ is dominated by three intuitive quantities (other than the constant additive term):
\begin{enumerate}
\item $\tlow(d - 1)$, which represents the number of rounds after which all lower-order models can be conclusively ruled out.
The dependence on $\tlow(d-1)$ is saying that this much mixability could have accumulated (due to poor approximability) before then.
\item $\thigh(D)$, which represents the amount of mixability the algorithm has to accumulate before performing effective higher-order model selection to rule out the overfitting models\footnote{It is also possible that the algorithm would not have accumulated even this mixability, and the model selection phase is never reached -- however, we never observed this case empirically.}.
\item $2^d \cdot \thigh(d)$, which represents the amount of mixability accumulated by the algorithm \textit{at the right model order}.
This is the term in analysis that corresponds to standard best-of-both-worlds analysis over a fixed model order.
\end{enumerate}

Now, we know from Equation~\eqref{eq:thigh} that $\thigh(h) = \frac{2}{(2\betastar - 1)^2} \ln \left(\frac{(D-d) \cdot 2^h}{(2\betastar - 1)^2 \epsilon}\right)$ and from Equation~\eqref{eq:tlow} that $\tlow(d-1) = \frac{32d}{\alpha_{d-1,d}^2} \left(d \cdot 2^{d-1} \ln 2 + \ln \left(\frac{64d}{\epsilon \alpha_{d-1,d}^2}\right)\right)$.
Substituting these in, we get
\begin{align}\label{eq:deltatfinal}
\Delta_T &= \Oh\left(2^d \left(\frac{d^2}{\alpha^2_{d-1,d}}\ln \left(\frac{d}{\alpha_{d-1,d}^2 \epsilon} \right) + \frac{D(d+2)}{\overline{\gamma}(2\betastar - 1)^2} \ln\left( \frac{D}{(2\betastar - 1)^2 \epsilon} \right)  \right)\right)
\end{align}

and substituting this into Lemma~\ref{lem:secondorderbound} gives
\begin{align}\label{eq:rtdfinal}
R_{T,d} &= \Oh\left(2^{2d} \left(\frac{d^2}{\alpha^2_{d-1,d}}\ln \left(\frac{d}{\alpha_{d-1,d}^2 \epsilon} \right) + \frac{D(d+2)}{\overline{\gamma}(2\betastar - 1)^2} \ln\left( \frac{D}{(2\betastar - 1)^2 \epsilon} \right)  \right)\right) ,
\end{align} 

completing the proof.
To highlight the dependence on true model order $d$ and maximum model order $D$ (as is expressed in the informal statement of Theorem~\ref{thm:contexttreeadahedge}), we can hide the constants in terms of parameters and write
\begin{align}\label{eq:rtdfinalbigoh}
R_{T,d} &= \Delta_T \left(1 + 2^d \right) \\
&= \Oh\left(2^{2d} \left(D \cdot d \cdot \ln \left(\frac{D}{\epsilon}\right)\right)\right) .
\end{align}

\section{Algorithmic benefits of \alg{ContextTreeAdaHedge$(D)$}}\label{sec:algbenefits}

In this section, we expound on the algorithmic benefits of \alg{ContextTreeAdaHedge$(D)$} equipped with prior function $g(\cdot)$: in particular, we formally show the reduced computational complexity of the algorithm, and the equivalence of the computationally efficient update in Equation~\eqref{eq:mainupdate} and the computationally naive update in Equation~\eqref{eq:naivemainupdate}.
The equivalence was originally proved for the multiplicative weights algorithm with a fixed learning rate~\cite{helmbold1997predicting}: here, we generalize the argument to include the family of exponential-weights updates with a time-varying, data-dependent learning rate.

\begin{proposition}
The runtime of \alg{ContextTreeAdaHedge$(D)$} per prediction round is $\Oh(2^D)$.
\end{proposition}

\begin{proof}
Consider round $t$ of prediction.
To carry out the efficient update in Equation~\eqref{eq:mainupdate}, we need to visit every node in the path of the context $X_t$. Since the full context is of length $D$, the update runs in $\Oh(D)$. To perform the prediction, we must calculate the probability distribution $\wvec_t$, which has $2$ entries. To calculate $\wvec_t$, we must visit every node in the single complete height $D$ tree to access the cumulative loss vectors $\{\Lvec_{x(D),t}\}_{x(D) \in \Xspace^D}$.

Since there are $2^D$ such loss vectors (i.e. $2^D$ nodes to visit), this operation takes $\Oh(2^D)$ time.
For a general prior, these cumulative contextual losses are accessed for every value of $h \in \{0,1,\ldots,D\}$.
Thus, the total computational complexity of performing an update is
\begin{align*}
\sum_{h = 0}^D 2^h = 2^{D+1} - 1 \in \mathcal{O}(2^D) .
\end{align*}

After performing prediction and receiving loss feedback, we need to access all these nodes again and update the cumulative losses.
By a similar argument as above, this is also a $\Oh(2^D)$ operation.
Therefore, the total computational compelexity per round is $\Oh(2^D)$.
\end{proof}

\paragraph{Computational complexity reduction: equivalence of updates}

Here, we state and prove the following proposition which shows equivalence of the naive update in Equation~\eqref{eq:naivemainupdate} and the computationally efficient update in Equation~\eqref{eq:mainupdate}.

\begin{proposition}\label{prop:algequivalence}
For any prior function $g: \{0,1,\ldots,D\} \to \reals_{+}$, the updates in Equation~\eqref{eq:mainupdate} and Equation~\eqref{eq:naivemainupdate} are equivalent.
\end{proposition}

\begin{proof}
It is convenient, for the purposes of this proof, to consider the overcounted set of tree experts ranging from orders $0$ to $D$.
In particular, any $d^{th}$-order tree expert is described by a function $f': \Xspace^d \to \Xspace$ and there are $2^{2^d}$ such experts.
Corresponding to prior function $g(\cdot)$, we set the initial distribution on tree experts:
\begin{align*}
\wtree_{1,\boldf} = \frac{\sum_{h=\text{order}(\boldf)}^D g(h)}{Z(g)}
\end{align*}

where $Z(g)$ is the initial normalizing factor, i.e. $Z(g) = \sum_{h=0}^D 2^{2^h} g(h)$.

Recall Equation~\eqref{eq:naivemainupdate} for the probability of choosing tree expert $\boldf$ at time $t$:
\begin{align*}
\wtree_{t,\boldf} = \frac{\left(\sum_{h=\text{order}(\boldf)}^D g(h)\right) e^{-\eta_t L_{t,\boldf}}}{Z_t(g)}
\end{align*}

where 
\begin{align*}
Z_t(g) := \sum_{\boldf \in \mathcal{F}_D} \left(\sum_{h=\text{order}(\boldf)}^D g(h)\right) e^{-\eta_t L_{t,\boldf}} .
\end{align*}

Also recall Equation~\eqref{eq:mainupdate} for the probability of $y \in \Xspace$ at time $t$:
\begin{align*}
w_{t,y} = \frac{\sum_{h=0}^D g'(h;\eta_t) e^{-\eta_t L_{X_t(h),t,y}}}{\sum_{h=0}^D g'(h;\eta_t) \left(\sum_{y \in \Xspace} e^{-\eta_t L_{X_t(h),t,y}}\right)} \text{ where } \\
g'(h;\eta_t) &= g(h) \prod_{x(h) \neq X_t(h)} \left(\sum_{y \in \Xspace} e^{-\eta_t L_{x(h),t,y}}\right)
\end{align*}

To show equivalence, it clearly suffices to show for every $y \in \Xspace$ that
\begin{align}
\sum_{\boldf \in \mathcal{F}_D: \boldf(X_t) = y} \wtree_{t,\boldf} = w_{t,y} .
\end{align}

We have
\begin{align*}
\sum_{\boldf \in \mathcal{F}_D: f(X_t) = j} \wtree_{t,\boldf} &= \sum_{h=0}^D \sum_{\substack{f: \text{order}(f) = h \\ f: f(X_t(h)) = y}} \wtree_{t,\boldf} \\
&= \sum_{h=0}^D \sum_{\substack{f: \text{order}(f) = h \\ f: f(X_t(h)) = y}} \frac{g(h)}{Z_t(g)} \prod_{x(h) \in \Xspace^h } e^{-\eta_t L_{x(h),t,f(x(h))}}  \\
&= \sum_{h=0}^D \frac{g(h)}{Z_t(g)} e^{-\eta_t L_{X_t(h),t,y}} \prod_{x(h) \neq X_t(h)} \left( \sum_{y' \in \Xspace} e^{-\eta_t L_{x(h),t,y}} \right) \\
&= \frac{\sum_{h=0}^D g'(h;\eta_t) e^{-\eta_t L_{X_t(h),t,y}} }{Z_t(g)}
\end{align*}

where we have used the distributive law of multiplication over addition, and substituted the definition of $g'(h;\eta_t)$.
To complete the proof of equivalence, it remains to show that 
\begin{align}\label{eq:equivalence}
Z_t(g) = \sum_{h=0}^D g'(h;\eta_t) \left(\sum_{y \in \Xspace} e^{-\eta_t L_{X_t(h),t,y}}\right) .
\end{align}

We use the distributive law to get
\begin{align*}
Z_t(g) &:= \sum_{f \in \mathcal{F}_D} \left(\sum_{h = \text{order}(f)}^D g(h) \right) \prod_{x(h) \in \Xspace^h } e^{-\eta_t L_{x(h),t,f(x(h))}} \\
&= \sum_{h=0}^D g(h) \prod_{x(h) \in \Xspace^h} \left(\sum_{y \in \Xspace} e^{-\eta_t L_{x(h),t,y}} \right) .
\end{align*}

We also substitute the expression for $g'(h;\eta_t)$ to get 
\begin{align*}
\sum_{h=0}^D g'(h;\eta_t) \left(\sum_{y \in \Xspace} e^{-\eta_t L_{X_t(h),t,y}}\right) 
&= \sum_{h=0}^D g(h) \left(  \prod_{x(h) \neq X_t(h)} \left(\sum_{y \in \Xspace} e^{-\eta_t L_{x(h),t,y}}\right) \right) \left(\sum_{y \in \Xspace} e^{-\eta_t L_{X_t(h),t,y}} \right) \\
&= \sum_{h=0}^D g(h) \prod_{x(h) \in \Xspace^h} \left(\sum_{y \in \Xspace} e^{-\eta_t L_{x(h),t,y}} \right) .
\end{align*}

Thus, Equation~\eqref{eq:equivalence} holds.
This completes the proof of equivalence of algorithms.

\end{proof}

\section{Supplementary algebra}\label{sec:algebra}

In this section, we state a couple of supplementary algebraic statements (and prove them when necessary).

\begin{fact}\label{f:maxlessthansum}
For two quantities $B, C \geq 0$, we have $\max\{B,C\} \leq B + C$.
\end{fact}

\begin{fact}\label{f:quadraticformula}
For two numbers $B, C \geq 0$, 
\begin{align*}
x^2 - Bx - C \leq 0 \implies x \leq \sqrt{C} + B .
\end{align*}

This results from the quadratic formula, which gives us
\begin{align*} 
x &\leq \frac{B + \sqrt{B^2 + 4C}}{2} \\
&\leq \frac{B + B + 2\sqrt{C}}{2} = \sqrt{C} + B 
\end{align*}

where the last inequality is a consequence of 
\begin{align*}
a, b \geq 0 \implies \sqrt{a + b} \leq \sqrt{a} + \sqrt{b} .
\end{align*}
\end{fact}

\section{Extra simulations to illustrate model adaptivity}\label{sec:extra}

In this section, we provide a supplementary simulation to the ones in Figure~\ref{fig:simulations} to show the maximal extent of advantage that adaptivity to the model order can give us.
We examine the $0^{th}$-order stochastic model on $\{(X_t,Y_t)\}_{t \geq 1}$, that is, $Y_t \text{ i.i.d } \BER(0.7)$ and $Y_t$ is independent of $X_t$, and again compare three algorithms: the \textit{optimal online algorithm with oracle knowledge of this structure} (the greedy \alg{Follow-the-Leader}); uniform-prior \alg{ContextTreeAdaHedge$(D)$}, which adapts to stochasticity but not model order; and our two-fold adaptive algorithm, \alg{ContextTreeAdaHedge$(D)$} with the prior function $\gprop(\cdot)$.

Figure~\ref{fig:simulationsiid} shows the evolution of regret and cumulative loss of all three algorithms.
The advantage of adaptivity is even more stark in the simple iid case: \alg{ContextTreeAdaHedge$(D)$} with prior function $\gprop(\cdot)$ is very close in its performance to the greedy optimal Follow-the-Leader algorithm.
The \textit{disadvantage of adaptivity} is also very clearly illustrated: uniform-prior \alg{ContextTreeAdaHedge$(D)$} is hugely overfitting for this simple iid example.

\multifigureexterior{fig:simulationsiid}{Comparison of optimal greedy FTL, \alg{ContextTreeAdaHedge$(D)$} with uniform prior and prior function $\gprop(\cdot)$ (where $D = 8$); against iid structure, upto $T = 1500$ rounds.}
{
\subfig{0.3\textwidth}{Total loss as a function of $T$.}{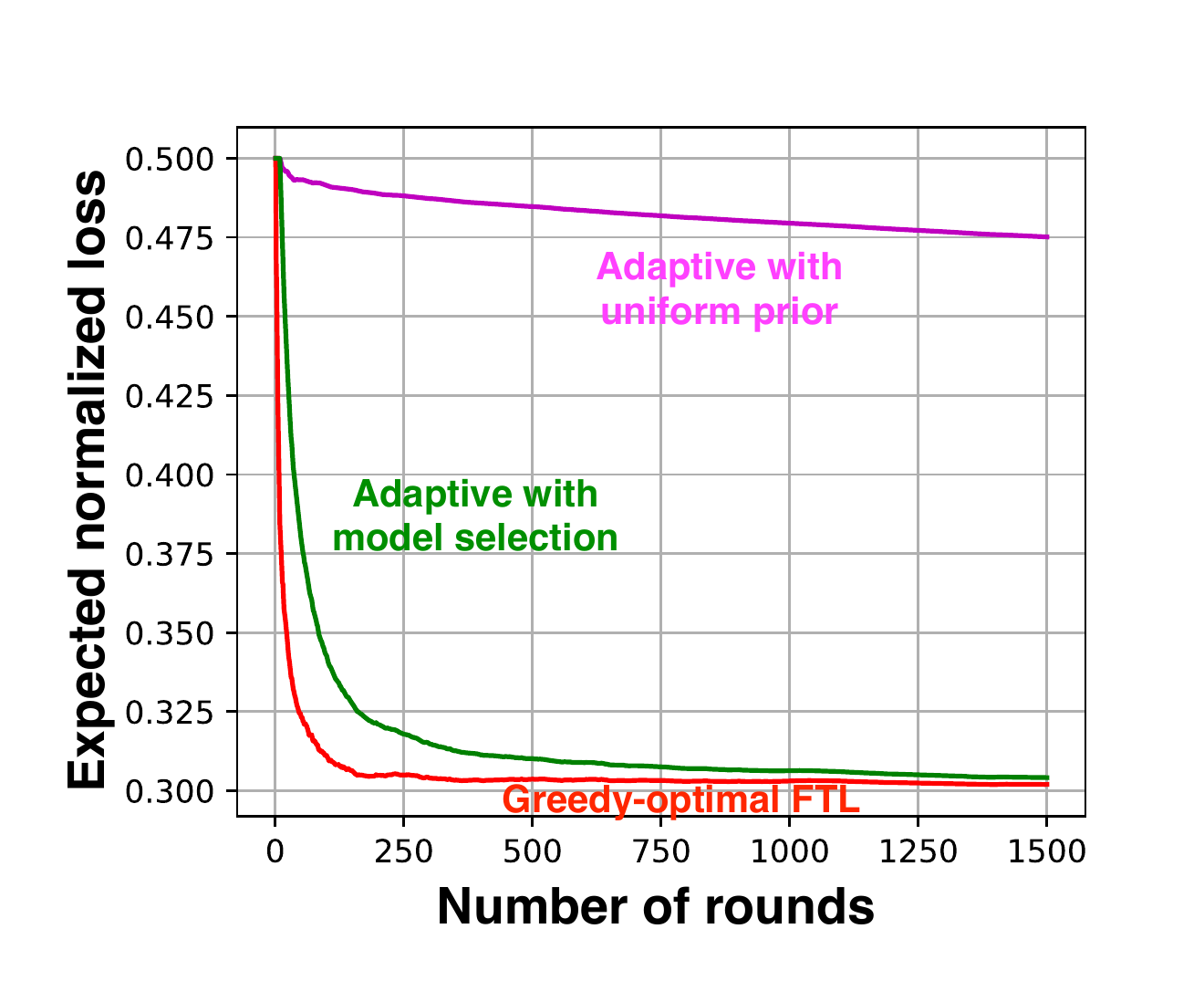}
\subfig{0.3\textwidth}{$R_{T,0}$ as a function of $T$.}{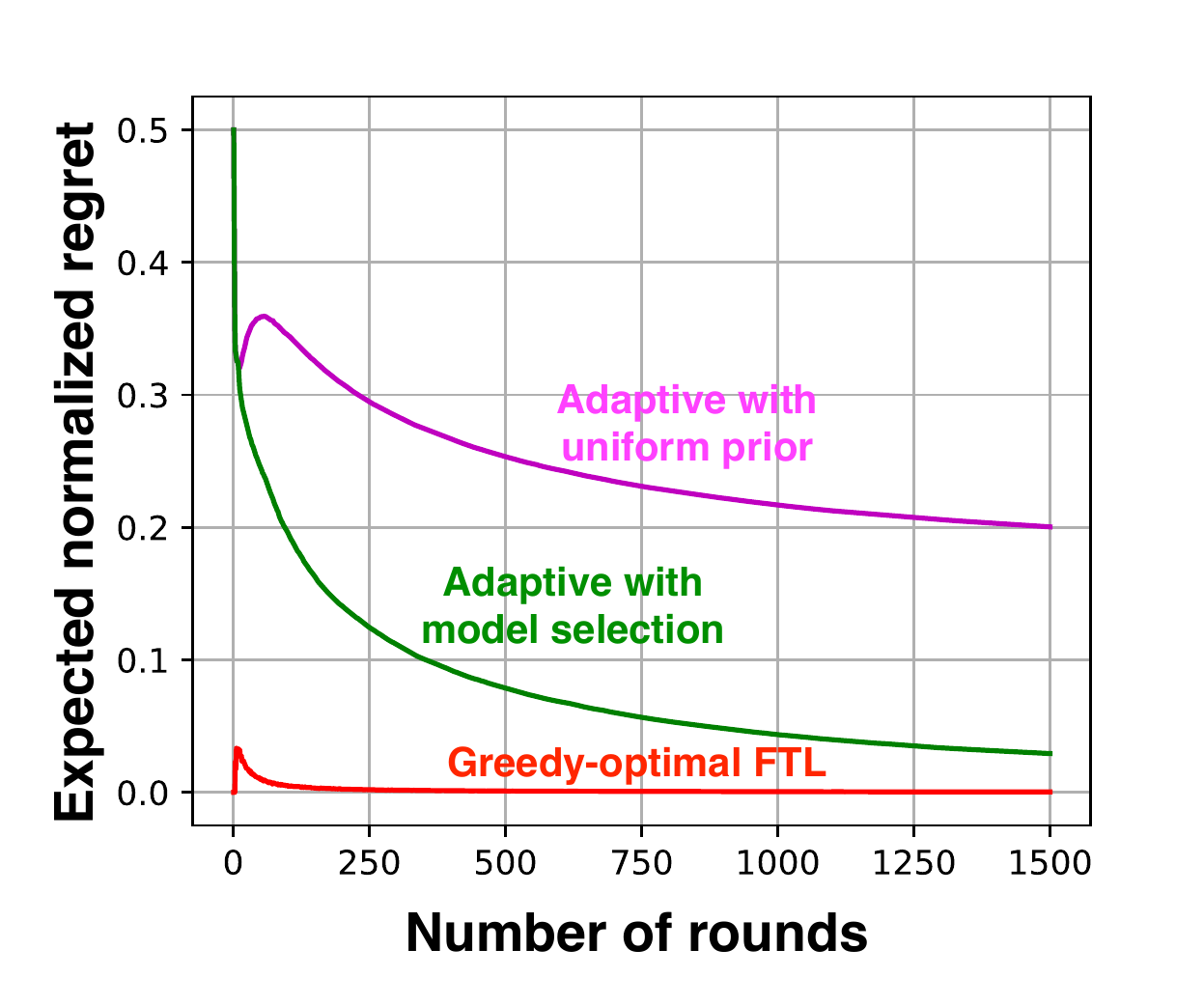}
}

\end{document}